\documentclass[11pt]{article}
\pdfoutput=1
\usepackage{microtype}
\usepackage{graphicx}
\usepackage{booktabs} 
\usepackage{natbib}
\usepackage{epsfig,epsf,fancybox}
\usepackage{amsmath}
\usepackage{mathrsfs}
\usepackage{amssymb}
\usepackage{color}
\usepackage{multirow}
\usepackage{paralist}
\usepackage{verbatim}
\usepackage{algorithm}
\usepackage{algorithmic}
\usepackage{galois}
\usepackage{boxedminipage}
\usepackage{accents}
\usepackage{stmaryrd}
\usepackage[bottom]{footmisc}
\usepackage{microtype}
\usepackage{booktabs} 
\usepackage{amsmath}
\usepackage{amssymb}
\usepackage{amsthm}
\usepackage[mathscr]{euscript}
\usepackage{nicefrac}
\usepackage{wrapfig}
\newcommand{\tablewhitespace}{\addlinespace[0.3em]}
\usepackage{xspace}
\usepackage{enumitem}
\usepackage{graphicx}
\usepackage{subcaption}
\usepackage{bm}

\usepackage{natbib}
\usepackage{tcolorbox}
\usepackage{url}
\usepackage{listings}

\usepackage[colorlinks, linkcolor={blue!70!black}, citecolor={blue!70!black}, urlcolor={blue!70!black}]{hyperref}

\usepackage{xcolor}

\usepackage{algorithm}
\usepackage{algorithmic}

\newtheorem{theorem}{Theorem}[section]
\newtheorem{proposition}[theorem]{Proposition}
\newtheorem{lemma}[theorem]{Lemma}
\newtheorem{corollary}[theorem]{Corollary}
\newtheorem{definition}[theorem]{Definition}
\newtheorem{assumption}[theorem]{Assumption}
\newtheorem{remark}[theorem]{Remark}
\newtheorem{condition}{Condition}[section]

\usepackage{mathtools}
\usepackage[flushleft]{threeparttable}
\usepackage{varwidth}
\usepackage{xfp, siunitx}
\usepackage{tikz}
\usepackage{bm}
\usepackage{tcolorbox}
\usepackage[mathscr]{euscript}

\definecolor{tab10_blue}{rgb}{0.121, 0.466, 0.705}
\definecolor{tab10_orange}{rgb}{1.0,   0.498, 0.054}
\definecolor{tab10_green}{rgb}{0.172, 0.627, 0.172}
\definecolor{tab10_red}{rgb}{0.839, 0.152, 0.156}
\definecolor{tab10_purple}{rgb}{0.580, 0.403, 0.741}
\definecolor{tab10_brown}{rgb}{0.549, 0.337, 0.294}
\definecolor{tab10_pink}{rgb}{0.890, 0.466, 0.760}
\definecolor{tab10_gray}{rgb}{0.498, 0.498, 0.498}
\definecolor{tab10_olive}{rgb}{0.737, 0.741, 0.133}
\definecolor{tab10_cyan}{rgb}{0.090, 0.745, 0.811}

\usepackage{enumitem}

\newcommand{\speedup}[1]{{\color{gray}(\ifdim #1 pt > 0.3pt #1\else $< #1$\fi{}$\times$)}}

\definecolor{Gray}{gray}{0.85}

\usepackage{enumitem}
\usepackage{bm}
\newcommand{\blambda}{\bm{\lambda}}
\newcommand{\ba}{\bm{a}}
\newcommand{\bc}{\bm{c}}
\newcommand{\bx}{\bm{x}}
\newcommand{\by}{\bm{y}}
\newcommand{\bz}{\bm{z}}

\newcommand{\bu}{\bm{u}}
\newcommand{\bs}{\bm{s}}

\newcommand{\bq}{\bm{q}}
\newcommand{\rD}{\mathrm{D}}
\newcommand{\cT}{\mathcal{T}}
\newcommand{\E}{\mathbb{E}}

\usepackage{bbold}

\numberwithin{equation}{section}

\DeclareMathOperator{\R}{\mathbb{R}}

\usepackage
[a4paper,
left=2.5cm,
right=2.5cm,
top=3cm,
bottom=4cm]
{geometry}

\newcommand{\algname}{{\sc Scaff-PD}}

\begin{document}


\begin{center}

{\bf{ {\LARGE \algname}: \LARGE Communication Efficient 
Fair and  Robust \\\vspace{2.5mm} Federated Learning}}

\vspace*{.25in}
{\large{
\begin{tabular}{c}
Yaodong Yu$^{\diamond}$ \quad Sai Praneeth Karimireddy$^{\diamond}$  \quad 
Yi Ma$^{\diamond}$ \quad
Michael I. Jordan$^{\diamond, \dagger}$
\end{tabular}
}}

\vspace*{.15in}
\begin{tabular}{c}
Department of Electrical Engineering and Computer Sciences$^\diamond$ \vspace{1.0mm}\\
Department of Statistics$^\dagger$ \vspace{1.5mm}\\ 
University of California, Berkeley
\end{tabular}

\vspace*{.1in}

\begin{abstract}
\noindent
We present {\algname}, a fast and communication-efficient algorithm for distributionally robust federated learning. Our approach improves fairness by optimizing a family of distributionally robust objectives tailored to heterogeneous clients. We leverage the special structure of these objectives, and design an accelerated primal dual (APD) algorithm which uses bias corrected local steps (as in {\sc Scaffold}) to achieve significant gains in communication efficiency and convergence speed. We evaluate {\algname} on several benchmark datasets and demonstrate its effectiveness in improving fairness and robustness while maintaining competitive accuracy. Our results suggest that {\algname} is a promising approach for federated learning in resource-constrained and heterogeneous settings.
\end{abstract}

\end{center}

\section{Introduction}\label{sec:intro}
Federated learning is a popular approach for training machine learning models on decentralized data, where data privacy concerns or other constraints prevent centralized data aggregation~\citep{mcmahan2017communication,kairouz2019federated}. In federated learning, model updates are computed locally on each device (the \emph{client}) and then aggregated to train a global model at the center (the \emph{server}). This approach has gained traction due to its ability to leverage data from multiple sources while preserving privacy, security, and autonomy, and has the potential to make machine learning more participatory in a range of interesting problem domains~\citep{paml2020,jones2020nonrivalry,pentland2021building}.

\vspace{0.05in}
Federated learning is naturally most attractive when the participating clients have access to different data, leading to data heterogeneity~\citep{du2022flamby}. 
This heterogeneity can lead to significant fairness issues, where the performance of the global model can be biased towards the data distribution of some clients over others~\citep{dwork2012fairness,li2019fair,abay2020mitigating}. Heterogeneity can also hurt the generalization of the global model~\citep{quinonero2008dataset,mohri2019agnostic}. Specifically, if some clients have a disproportionate influence on the global model, the resulting model is neither fair nor will it generalize well to new clients. Such disparities are especially prevalent and detrimental in medical research, and have resulted in misdiagnosis and suboptimal treatment~\citep{graham2015disparities,albain2009racial,nana2021health}.

\vspace{0.05in}
To address these challenges, distributionally robust objectives (DRO) explicitly account for the heterogeneity across clients and seek to optimize performance under the worst-case data distribution across clients, rather than just the average performance~\citep{rahimian2019distributionally}. This approach can lead to more robust models that are less biased towards specific clients and more likely to generalize to new clients~\citep{mohri2019agnostic,duchi2023distributionally}. However, such robust objectives are significantly harder to optimize. Current algorithms have very slow convergence, potentially to the point of being impractical~\citep{ro2021communication}. This leads to the central question of our work:
\begin{center}
\parbox{0.9\textwidth}{\emph{Can we design federated optimization techniques for the DRO problem with convergence rates that match their average objective counterparts?}}
\vspace{-0.1in}
\end{center}

\begin{figure}[t]
    \vspace{-0.12in}
     \centering
     \includegraphics[width=0.9\textwidth]{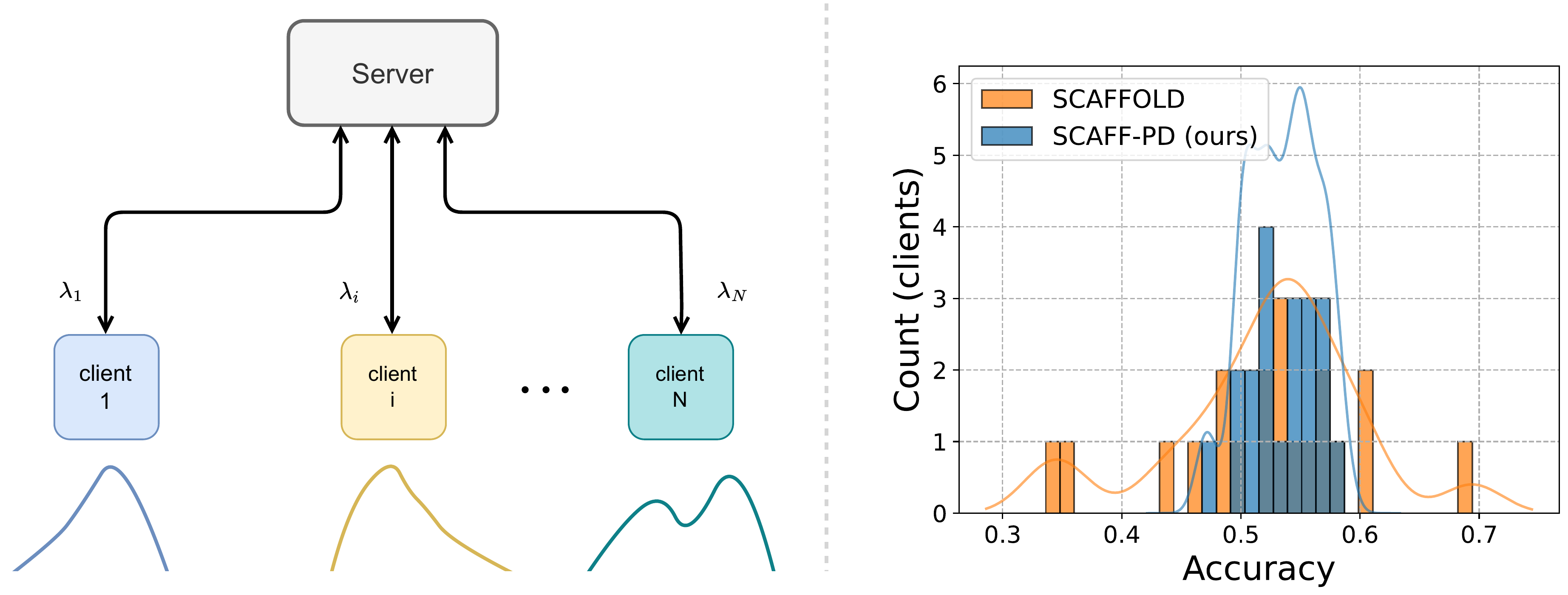}
    \caption{
    (\textbf{left}) In federated learning, the data distribution across individual clients differ significantly from one another. 
    (\textbf{right}) When directly applying SOTA federated optimization algorithm ({\sc Scaffold}), the learned global model is biased toward certain clients, leading to noticeably worse performance when applied to a subset of participating clients. Our proposed algorithm---{\algname}---largely mitigates this bias via learning a distributionally robust global model, which significantly enhances the performance of the most challenging subset of clients, specifically the worst 20\%. 
    }
    \label{fig:fig1}
    \vspace{-0.1in}
\end{figure}

\subsection{Our Contributions}
We summarize our contributions below.
\vspace{-0.1in}
\paragraph{Framework.} We present a general formulation for the cross-silo federated DRO problem:
\begin{equation}\label{eq:dro-def}
    \min_{\bx} \max_{\blambda \in \Lambda} \left\{ F(\bx, \blambda) := \sum_{i=1}^N\lambda_i\cdot f_{i}(\bx) - \psi(\blambda)\right\},
\end{equation}
where $f_i(\bx)$ is the loss suffered by client $i$. Instead of minimizing a simple average of the client losses, equation \eqref{eq:dro-def} incorporates weights using $\blambda \in \R^N$. The choice of $\blambda$ is made in a \emph{worst-case} manner, while being subject to the constraint set $\Lambda$ and regularized with $\psi(\blambda)$. As we will show, this formulation is a generalization of several specific fair objectives that have been proposed in the federated learning literature~\citep{mohri2019agnostic,li2019fair,li2020tilted,zhang2022proportional,pillutla2021federated}.

\vspace{-0.1in}
\paragraph{Algorithm.} The objective defined in equation \eqref{eq:dro-def} is a min-max problem and can be directly optimized using well-established algorithms such as gradient descent ascent (GDA). However, such approaches ignore the unique structure of our formulation, particularly the linearity of the interaction term between $\blambda$ and $\bx$. We leverage this to design an accelerated primal-dual (APD) algorithm~\citep{hamedani2021primal}. Additionally, we propose to use control variates (\`a~la {\sc Scaffold}) to correct the bias caused by local steps, making optimal use of local client computation \citep{karimireddy2020scaffold}. Our proposed method, {\algname}, combines these ideas to provide an efficient and practical algorithm, compatible with secure aggregation.

\vspace{-0.1in}
\paragraph{Convergence.} We provide strong convergence guarantees for {\algname} when ${f_i}$ are strongly convex. If $\psi(\blambda)$ is a generic convex function, we achieve an {accelerated} $O\left( {1}/{T^2} \right)$ rate of convergence. Furthermore, if $\psi$ is strongly convex, {\algname} converges linearly at a rate of $\exp\left(-O(T)\right)$. This represents the first federated approach for the DRO problem that achieves \emph{linear} convergence, let alone an \emph{accelerated} rate.
Finally, we extend our analysis to the stochastic setting, where we obtain an optimal rate of $O\left( {1}/{T} \right)$, and improve over the previous $O(1/\sqrt{T})$ rate. Thus, we show that the sample complexity as well as the communication complexity for the DRO problem matches that of the easier average objective.

\paragraph{Practical Evaluation.} We conducted comprehensive simulations and demonstrate accelerated convergence, robustness to data heterogeneity, and the ability to leverage local computations.

For deep learning models, we avail ourselves of a two-stage Train-Convexify-Train method~\citep{yu2022tct}. First, we train a deep learning model using conventional federated learning methods, such as FedAvg. Then, we apply {\algname} to fine tune a convex approximation. To evaluate our algorithms, we use several real-world datasets with various distributionally robust objectives, and we study the trade-off between the mean and tail accuracy of these methods.

\section{Related Work}\label{sec:relatedwork}

\vspace{-0.05in}
\paragraph{Cross-silo FL.} Federated learning (FL) is a distributed machine learning paradigm that enables model training without exchanging raw data. 
In cross-silo FL (which is our focus), valuable data is split across different organizations, and each organization is either protected by privacy regulations or unwilling to share their raw data. Such organizations are referred to as ``data islands'' and can be found in hospital networks, financial institutions, autonomous-vehicle companies, etc. Thus, cross-silo FL involves a few highly reliable clients who potentially have extremely diverse data.

\vspace{0.05in}
The most widely used federated optimization algorithm is Federated Averaging (FedAvg) \citep{mcmahan2017communication}, which averages the local model updates to produce a global model. However, FedAvg is known to suffer from poor convergence when the local datasets are heterogeneous ~\cite[etc.]{hsieh2020non,li2020federated,karimireddy2020scaffold,reddi2021adaptive,wang2021field,du2022flamby}. Scaffold~\citep{karimireddy2020scaffold} corrects for this heterogeneity, leading to more accurate updates and faster convergence~\citep{mishchenko2022proxskip,li2022partial,yu2022tct}. However, all of these methods are restricted to optimizing the average of the client objectives.

\vspace{-0.1in}
\paragraph{Distributionally Robust Optimization.} DRO is a framework for optimization under uncertainty, where the goal is to optimize the worst-case performance over a set of probability distributions. See \citet{rahimian2019distributionally} for a review and its history in risk management, economics, and finance. Fast centralized optimization methods have been developed when uncertainity is represented by $f$-divergences \citep{wiesemann2014distributionally,namkoong2016stochastic,levy2020large} or Wasserstein distances \citep{mohajerin2018data,gao2022distributionally}. The former approach accounts for changing proportions of subpopulations, relating it to notions of subpopulation fairness~\citep{duchi2023distributionally,santurkar2020breeds,piratla2021focus,martinez2021blind}. Our work also implicitly focuses on $f$-divergences. \citet{deng2020distributionally} and \citet{zecchin2022communication} adapt the gradient-descent-ascent (GDA) algorithm to solve the federated and decentralized DRO problems respectively. However, these methods inherit the slowness of both the GDA and FedAvg algorithms, making their performance trail the state of the art for the average objective~\citep{mishchenko2022proxskip}.

\vspace{-0.1in}
\paragraph{Fairness in FL.} While fairness is an extremely multi-faceted concept, here we are concerned with the distribution of model performance across clients. \citet{mohri2019agnostic} noted that minimizing the average of the client losses may lead to unfair distribution of errors, and instead proposed an \emph{agnostic FL} (AFL) framework which minimizes a worst-case mixture of the client losses. Alternatives and extensions to AFL have also been  proposed subsequently~\citet{li2019fair,li2020tilted,pillutla2021federated}. Again, the convergence of optimization methods for these losses (when analyzed) is significantly slower than their centralized counterparts.

\vspace{0.05in}
While all of these works demand equitable performance across all clients, others propose to scale a client's accuracy in proportion to their contribution~\citep{sim2020collaborative,blum2021one,xu2021gradient,zhang2022proportional,karimireddy2022mechanisms}. These methods are motivated by game-theoretic considerations to incentivize clients and improve the quality of the data contributions. Our framework \eqref{eq:dro-def} can be applied to such mechanisms by an appropriate choice of $\{f_i\}, \Lambda, \text{and } \psi$. For example, \citet{zhang2022proportional} show how to set these to recover the Nash bargaining solution~\citep{nash1950bargaining}. Thus, our work can be seen as a practical optimization algorithm to implement many of the mechanisms studied in FL.

\vspace{0.05in}
Finally, personalization---serving a separate model to each client---has also been proposed as a method to improve the distribution of client performance~\citep{yu2020salvaging}. However, personalized models are sometimes not feasible either due to regulations~\citep{vokinger2021continual} or because the client may not have additional data. Further, personalization does not remove the differences in performance (though it does reduce it)~\citep{yu2020salvaging}, nor does it solve the game-theoretic considerations described above. Extending our work to this setting is an important question we leave for future work.

\vspace{-0.05in}
\section{Problem Setup}
We consider the min-max optimization problem  in the context of federated learning, where the objective function, defined in Eq.~\eqref{eq:dro-def}, is distributed among $N$ clients. Each $f_i: \R^{d}\rightarrow\R$ is the local function on the $i$-th client, where $f_i(\bx) = \E_{\xi \sim \mathcal{D}_i}[f(\bx, \xi)]$ and $\mathcal{D}_i$ is the data distribution of the $i$-th client. For example, we can define $\mathcal{D}_i$ as the uniform distribution over the training dataset present on the $i$-th client.

\vspace{0.05in}
\textbf{Notation.} We use the notation $\bx^{r} \in \R^{d}$ to denote the global iterate at the $r$-th round, and use $\bu_{i,j}^{r} \in \R^{d}$ to denote the local iterate at the $j$-th step on the $i$-th client (at the $r$-th round).
We apply $\blambda = [\lambda_1, \dots, \lambda_N]^{\top} \in \R^{N}$ to denote the weight vector, where $\lambda_i$ is the weight for client $i$. 
We let $[N]$ denote the set $\{1, \dots, N\}$. 
To facilitate clarity and presentation, we let 
\begin{equation}
\Phi(\bx, \blambda)=\sum_{i=1}^N\lambda_i\cdot f_{i}(\bx).
\end{equation} 
iFor local gradients, we let $g_i(\bu_{i,j-1})$ denote the stochastic gradient of $f_i$ at iterate $\bu_{i,j-1}$:
\begin{equation}\label{eq:local-gradient-def}
    g_i(\bu^{r}_{i,j-1}) = \nabla f_i (\bu^{r}_{i,j-1}, \xi^{r}_{i, j-1}).
\end{equation}

\textbf{Choosing $\psi$ and $\Lambda$.}
We let $\psi: \R^{N} \rightarrow \R$ denote the regularization on the weight vector $\blambda$. 
The $\chi^{2}$ penalty~\citep{levy2020large} involves setting 
\begin{equation}\label{eq:psi-function-def}
    \psi(\blambda) = \mathrm{D}_{\chi^{2}}(\blambda) = \frac{\rho}{2N}\sum_{i=1}^{N}(N\lambda_i - 1)^{2}\text{, and } \Lambda = \Delta^N\,.
\end{equation}
When regularization is set to zero with $\rho =0$, the DRO formulation \eqref{eq:dro-def} recovers the agnostic federated learning (AFL) of~\citet{mohri2019agnostic}. A non-zero value of $\rho$ can be used to trade off the worst-case loss against the average loss. In particular, setting $\rho \rightarrow \infty$ recovers the standard average FL objective. While we will primarily focus on~\eqref{eq:psi-function-def} in this work, other choices are also possible. The DRO objective becomes the $\alpha$-Conditional Value at Risk (CVaR) loss~\citep{duchi2021learning}, also known as super-quantile loss~\citep{pillutla2021federated} by setting
\[
\psi(\blambda) =0\text{, and } \Lambda = \{\blambda \in \Delta, \lambda_i \leq 1/(\alpha N)\}\,.
\]
Finally, we can recover the Q-FL loss of~\citet{li2019fair} by setting 
\[
\psi(\blambda) = \|\blambda\|^{1 + \tfrac{1}{q}}\text{, and } \Lambda = \R^N \,.
\]

\vspace{-0.1in}
\textbf{Definitions and assumptions.} 
In the convergence analysis of our proposed algorithms, we rely on the following definitions and assumptions regarding the local functions and the regularization term $\psi$:
\begin{definition}[Smoothness]\label{definition:f-smooth}
    $f(\cdot)$ is convex and differentiable, and there exists $L \geq 0$ such that for any $\bx_1, \bx_2$ in the domain of $f_i(\cdot)$, 
    \begin{equation}
        \|\nabla f_i(\bx_1) - \nabla f_i(\bx_2)\| \leq L\|\bx_1 - \bx_2\|.
    \end{equation}
\end{definition}

\vspace{-0.1in}
\begin{definition}[Strong convexity]\label{definition:f-strong-convex}
    $f(\cdot)$ is $\mu$-strongly convex, i.e., 
    \begin{equation}
        f(\bx_2) \geq f(\bx_1) + \langle \nabla f(\bx_1), \bx_2 - \bx_1 \rangle + \frac{\mu}{2}\|\bx_2 - \bx_1\|^{2}.
    \end{equation}
\end{definition}

\begin{assumption}[Smoothness w.r.t. $\Phi$]\label{assumption:L_yx-smooth}
    $\Phi(\bx, \cdot)$ is concave and differentiable, and there exists $L_{\blambda\bx} \geq 0$ such that for any $\bx_1, \bx_2$ in the domain of $\Phi(\cdot, \blambda)$ and $\blambda_1, \blambda_2$ in the domain of $\Phi(\bx, \cdot)$, 
    \begin{equation}
        \|\nabla_{\blambda}\Phi(\bx_1, \blambda_1) - \nabla_{\blambda}\Phi(\bx_2, \blambda_2)\| \leq L_{\blambda\bx}\|\bx_1 - \bx_2\|.
    \end{equation}
\end{assumption}

\begin{assumption}[Bounded noise]\label{assumption:local-noise}
There exist $\zeta \geq 0$ such that for all $i\in[N]$, the local gradient $g_i(\bx)$ defined in Eq.~\eqref{eq:local-gradient-def} satisfies 
    \begin{equation}
        \E\left[\|g_i(\bx) - \nabla f_i(\bx)\|^{2}\right] \leq \zeta^{2}, \quad \E\left[g_i(\bx)\right] = \nabla f_i(\bx).
    \end{equation}
\end{assumption}

\section{{\algname}: Accelerated Primal-Dual Federated Algorithm with Bias Corrected Local Steps }\label{sec:alg}

In this section, we describe our proposed algorithm {\algname} (Stochastic Controlled Averaging with Primal-Dual updates) for solving the 
federated DRO problem \eqref{eq:dro-def}. 
We present the pseudo-code for {\algname} in Algorithm~\ref{Algorithm:ScaffoldAPD} and algorithm used  for local updates in Algorithm~\ref{Algorithm:local-ScaffoldAPD}.

\vspace{-0.1in}
\begin{algorithm}[ht]
\caption{{\algname}($\bx^{0}, \blambda^{0}$)}\label{Algorithm:ScaffoldAPD}
\begin{algorithmic}
\FOR{$r = 1, 2, \ldots, R$}
    \STATE { \texttt{\# (1).\,Collect gradient and loss vector}}
    \STATE Set parameters $\{\tau_{r}, \sigma_{r},  \gamma_{r}, \theta_{r}\}$
    \FOR{$i = 1, 2, \ldots, N$}
    \STATE $L_i^{r} = f_i(\bx^r)$, \,\,$\bc_i^{r} = g_i(\bx^r)$, \,\,\text{Communicate $(L_i^{r}, \bc_i^{r})$ to center}
    \ENDFOR
    \STATE { \texttt{\# (2).\,Update dual $\blambda$}}
    \STATE {
    {
    \begin{small}
    \begin{equation}\label{eq:alg-update-dual}
    \begin{aligned}
    \bs^r &= (1+\theta_{r}) \nabla_{\blambda}\Phi(\bx^{r}, \blambda^{r})  - \theta_{r}\nabla_{\blambda}\Phi(\bx^{r-1}, \blambda^{r-1})\\
    \blambda^{r+1} &= \text{argmin}_{\blambda \in \Lambda}\left\{\psi(\blambda)-\langle\bs^{r}, \blambda\rangle + \frac{1}{\sigma_r}{\rD}(\blambda, \blambda^{r})\right\}
    \end{aligned}
    \end{equation}
    \end{small}
    }
    }
    \STATE {\texttt{\# (3).\,Update primal $\bx$}}
    \STATE $\bc^r = \sum_{i=1}^{N}{\lambda_i^{r+1}}\bc_i^{r}$, \,\, \text{Communicate $\bc^r$ to each client}
    \FOR{$i = 1, 2, \ldots, N$}
    \STATE {$\Delta\bu_i^r \leftarrow$ {\sc Local-update}($\bx^{r}, \bc_i^{r}, \bc^{r}$)}, \,\, \text{Communicate $\Delta\bu_{i}^{r}$ to the center}
    \ENDFOR
    \STATE {Aggregate updates from different client via the weight vector $\blambda^{r+1}$}
    \begin{small}
    \begin{equation}\label{eq:alg-update-primal}
    \begin{aligned}
    \bx^{r+1} = \text{argmin}_{\bx}\Bigg\{\big\langle \sum_{i=1}^{N}\lambda_i^{r+1}\Delta\bu^{r}_i, \bx\big\rangle + \frac{1}{\tau_r}\mathrm{D}(\bx, \bx^{r})\Bigg\}
    \end{aligned}
    \end{equation}
    \end{small}
\ENDFOR
\STATE \textbf{Return:} $(\bx^{R+1}$, $\blambda^{R+1})$
\end{algorithmic}
\end{algorithm}

\vspace{0.1in}
As described in Algorithm~\ref{Algorithm:ScaffoldAPD}, {\algname} comprises three main steps that are executed at each communication round $r$: 
(1). Collecting  loss vector $[L_1^{r}, \dots, L_N^{r}]^{\top}$ and gradients $\{g_i(\bx^{r})\}_{i=1}^{N}$ (for bias correction); 
(2). Update to the dual variable  by Eq.~\eqref{eq:alg-update-dual}; 
(3). Local updates to each client model, and aggregating the updates by using the updated dual variable, i.e., Eq.~\eqref{eq:alg-update-dual}. We provide the pseudo-code for local updates in Algorithm~\ref{Algorithm:local-ScaffoldAPD}.

\begin{algorithm}[ht]
\caption{{\sc Local-update}$(\bx, \bc_i, \bc)$}\label{Algorithm:local-ScaffoldAPD}
\begin{algorithmic}
\STATE \textbf{Input:} optimization parameters $(\eta_\ell, J)$, model parameters $(\bc_i, \bc$, $\bx)$
    \STATE $\bu_{i, 0} = {\bx}$
    \FOR{$j = 1, 2, \ldots, J$}
    \STATE ${\bu_{i, j} = \bu_{i, j-1} -  \eta_{\ell}\cdot\left(g_i(\bu_{i, j-1})   - \bc_i + \bc \right)} $
    \ENDFOR
    \STATE $\Delta\bu_{i} = (\bx - \bu_{i, J})/(\eta_{\ell}J)$
\STATE \textbf{Return:} $\Delta\bu_{i}$
\end{algorithmic}
\end{algorithm}

\vspace{0.05in}
\textbf{Extrapolated Dual Update.}
Based on the computed loss vector $\nabla_{\blambda}\Phi(\bx^{r}, \blambda^{r}) = [L_1^{r}, \dots, L_N^{r}]^{\top}$ in the first step, we update the weight vector $\blambda$. Importantly, when $\theta_r > 0$, we use both the dual gradient from the current 
round ($\nabla_{\blambda}\Phi(\bx^{r}, \blambda^{r})$) as well as the past round ($\nabla_{\blambda}\Phi(\bx^{r-1}, \blambda^{r-1})$) to obtain the extrapolated gradient $\bs^{r}$. The gradient extrapolation step is widely used in  primal-dual hybrid gradient (PDHG) methods~\citep{chambolle2016ergodic} for solving convex-concave saddle-point problems, and it provides the key component in our algorithm for achieving acceleration. 
The extrapolation step used in Eq.~\eqref{eq:alg-update-dual} is to Nesterov’s acceleration~\citep{nesterov2003introductory}, which can lead to faster convergence rate and has been widely utilized for achieving acceleration in solving various optimization problems.~\citep{chambolle2011first, chambolle2016ergodic, zhang2015stochastic, hamedani2021primal}. 

\vspace{0.05in}
\textbf{Local Steps and Control Variates $\bc_i$.} 
Supposing that communication is not a limiting factor, each client can compute its local gradient and transmit it to the server without any local steps. In this case, the update to the primal variable $\bx$ becomes
\vspace{-0.1in}
\begin{equation}\label{eq:update-x-central-case}
    \Delta \bu_i^{r} = g_i(\bx^{r}), \quad \text{argmin}_{\bx}\Big\{\langle \sum_{i=1}^{N}\lambda_i^{r+1}g_i(\bx^{r}), \bx\rangle + \frac{1}{\tau_r}\mathrm{D}(\bx, \bx^{r})\Big\}.
\end{equation}
This update performs the primal update with the unbiased gradient $\nabla_{\bx}F(\bx^{r}, \blambda^{r+1})$, which is equivalent to the standard primal update in primal-dual-based algorithms~\citep{chambolle2016ergodic,hamedani2021primal,zhang2022sapd}. 
However, such an update does not effectively utilize the local computational resources available on each client. Hence, we would like to perform multiple local update steps. The catch is that performing multiple local steps is known to lead to biased updates and ``client-drift''~\citep{karimireddy2020scaffold,woodworth2020minibatch,wang2020tackling}. We explicitly correct for this bias using control variates $\{\bc_{i}\}_{i\in [N]}$ similar to {\sc Scaffold}. 
As we will demonstrate in the subsequent theoretical analysis, this correction allows {\algname} to converge to the saddle-point solution of the DRO problem regardless of the data heterogeneity.

\vspace{0.05in}
While we use local updates on the primal variable, we do not perform any on the dual variable. This is unlike general federated min-max optimization algorithms~\citep{hou2021efficient, beznosikov2022decentralized}.
This design aligns well with the federated DRO formulation  since it is impractical for each client to update the weight vector at each local step due to their lack of knowledge regarding the loss values of other clients. 
The aggregation of {\algname} on the server resembles federated algorithms used for solving minimization problems, with the key difference being the utilization of the updated weight vector for primal aggregation.

\section{Theoretical Analysis}\label{sec:theory}
We now present the convergence results for {\algname} in solving the min-max optimization problem described in Eq.~\eqref{eq:dro-def}. 
Firstly, in Section~\ref{sec:theory-sc-c}, we introduce the results for the strongly-convex-concave setting. Subsequently, in Section~\ref{sec:theory-sc-sc}, we present the results for the strongly-convex-convex setting. 
\vspace{-0.1in}
\subsection{Strongly-convex-concave Setting}\label{sec:theory-sc-c}
We first introduce how to choice the parameters for {\algname} in when $\psi$ is convex and $\{f_i\}_{i\in[N]}$ are strongly convex in Condition~\ref{condition:stepsize-scc}.

\begin{condition}\label{condition:stepsize-scc}
The parameters of Algorithm~\ref{Algorithm:ScaffoldAPD} are defined as
\begin{equation}
\sigma_{-1} = \gamma_0 \bar{\tau},\quad \sigma_r = \gamma_r \tau_r, \quad
\theta_r = \sigma_{r-1}/\sigma_r,\quad
\gamma_{r+1} = \gamma_r (1 + \mu_{\bx}\tau_{r}).
\end{equation}
\end{condition}
\vspace{-0.05in}
Next we present our convergence results in this setting.

\begin{theorem}\label{thm:main-sc-c}
Supppose $\{f_i\}_{i\in[N]}$ are $\mu_{\bx}$-strongly convex. If  Assumption~\ref{assumption:L_yx-smooth} and Assumption\ref{assumption:local-noise} hold,  and we let the parameters $\{\tau_{r}, \sigma_{r},  \gamma_{r}, \theta_{r}\}$ of Algorithm~\ref{Algorithm:ScaffoldAPD} satisfy Condition~\ref{condition:stepsize-scc}, then the $R$-th iterate $(\bx^{R}, \blambda^{R})$ satisfies
\begin{equation}
    \E\left[\|\bx^{R} - \bx^{\star}\|^{2} \right] 
    \leq \frac{C_1}{R^{2}}\left[\|\bx^{\star} - \bx^{0}\|^{2} + \| \blambda^{0} - \blambda^{\star}\|^{2}\right] + \frac{C_2}{R}\zeta^{2},
\end{equation}
where $C_1, C_2 \geq 0$ are non-negative constants.
\end{theorem}

\vspace{-0.1in}
\begin{corollary}\label{corollary:sc-c}
Under the assumptions in Theorem~\ref{thm:main-sc-c},
\begin{itemize}[leftmargin=0.7cm]
    \item (deterministic local gradient): If the local gradient satisfies $g_i(\bx)=\nabla f_i(\bx)$ for $i\in[N]$, then after
    $        O\left(\frac{\|\bx^{\star} - \bx^{0}\|^{2} + \| \blambda^{0} - \blambda^{\star}\|^{2}}{\sqrt{\varepsilon}}\right)$
    rounds, we have $\|\bx^{R} - \bx^{\star}\|^{2} \leq \varepsilon$.
    \item (stochastic local gradient): If the local gradient satisfies Assumption~\ref{assumption:local-noise} with $\sigma > 0$, then after $O\left(\frac{\|\bx^{\star} - \bx^{0}\|^{2} + \| \blambda^{0} - \blambda^{\star}\|^{2}}{\sqrt{\varepsilon}} + \frac{\zeta^{2}}{\varepsilon}\right)$
    rounds, we have $\E\left[\|\bx^{R} - \bx^{\star}\|^{2}\right] \leq \varepsilon$.
\end{itemize}
\end{corollary}
\vspace{-0.15in}
\begin{remark}
As suggested by the Corollary~\ref{corollary:sc-c}, in the deterministic setting ($\zeta=0$, when applying {\algname} for solving the min-max problems in the vanilla AFL and the super-quantile approach, {\algname} achieves the convergence rate of $O(1/R^2)$.
The rate of {\algname} is faster than existing algorithms --
the convergence rate  is $O(1/R)$ in both \citet{mohri2019agnostic, pillutla2021federated}. 
In addition, the algorithm with theoretical convergence guarantees introduced in \citet{mohri2019agnostic} does not apply local steps (i.e., number of local updates $J=1$), resulting in inferior performance in practical applications.
\end{remark}
\vspace{-0.15in}
\begin{remark}
{\algname} matches the rates ($O(1/R^{2})$) of the centralized accelerated primal-dual algorithm~\citep{hamedani2021primal} when $\zeta=0$. Meanwhile, our proposed algorithm converges faster compared to directly applying centralized gradient descent ascent (GDA) and extra-gradient method (EG) for solving Eq.~\eqref{eq:dro-def}, which achieve a rate of $O(1/R)$. 
\end{remark}

\subsection{Strongly-convex-strongly-concave Setting}\label{sec:theory-sc-sc}
We next present results for the strongly-convex-strongly-concave setting. Differing from the strongly-convex-concave setting, the parameters of Algorithm~\ref{Algorithm:ScaffoldAPD} are fixed across different rounds, as follows.
\begin{condition}\label{condition:stepsize-scsc}
The parameters of Algorithm~\ref{Algorithm:ScaffoldAPD} are defined as
\begin{small}
\begin{equation}
\begin{aligned}
    \mu_{\bx}\tau = O\left(\frac{1-\theta}{\theta}\right),\quad \mu_{\blambda}\sigma = O\left(\frac{1-\theta}{\theta}\right),\quad \frac{1}{1-\theta} = O\left(\Bigg(\frac{L_{\bx\bx}}{\mu_{\bx}} + \sqrt{\frac{L_{\blambda\bx}^{2}}{\mu_{\bx}\mu_{\blambda}}}\,\Bigg) \vee \frac{\zeta^{2}}{\mu_{\bx}\varepsilon}\right).
\end{aligned}
\end{equation}
\end{small}
\end{condition}

\begin{theorem}\label{thm:main-sc-sc}
Suppose $\{f_i\}_{i\in[N]}$ are $\mu_{\bx}$-strongly convex and $\psi$ is $\mu_{\by}$-strongly convex. 
If  Assumption~\ref{assumption:L_yx-smooth} and Assumption\ref{assumption:local-noise} hold,  and we let the parameters $\{\tau, \sigma, \theta\}$ of Algorithm~\ref{Algorithm:ScaffoldAPD} satisfy Condition~\ref{condition:stepsize-scsc}, then the $R$-th iterate $(\bx^{R}, \blambda^{R})$ satisfies 
\begin{equation}
\E\left[\mu_{\bx}\|\bx^{r} - \bx^{\star}\|^{2}\right] \leq C_1\theta^{R}\left[\|\bx^{0} - \bx^{\star}\|^{2} + \| \blambda^{0} - \blambda^{\star}\|^{2}\right]+ C_2(1-\theta)\frac{\zeta^{2}}{\mu_{\bx}},
\end{equation}
where $C_1, C_2 \geq 0$ are non-negative constants.
\end{theorem}
\begin{corollary}
Under the assumptions in Theorem~\ref{thm:main-sc-sc},
\begin{itemize}[leftmargin=0.7cm]
    \item (deterministic local gradient): If the local gradient satisfies $g_i(\bx)=\nabla f_i(\bx)$ for $i\in[N]$, then after 
    $O\left(\Bigg(\frac{L_{\bx\bx}}{\mu_{\bx}} + \sqrt{\frac{L_{\blambda\bx}^{2}}{\mu_{\bx}\mu_{\blambda}}}\Bigg)\log\left(\frac{\|\bx^{0} - \bx^{\star}\|^{2} + \| \blambda^{0} - \blambda^{\star}\|^{2}}{\varepsilon}\right)\right)$ 
    rounds, $\mu_{\bx}\|\bx^{R} - \bx^{\star}\|^{2} \leq \varepsilon$.
    \item (stochastic local gradient): If the local gradient satisfies Assumption~\ref{assumption:local-noise} with $\zeta > 0$, then after 
    $O\left(\Bigg(\frac{L_{\bx\bx}}{\mu_{\bx}} + \sqrt{\frac{L_{\blambda\bx}^{2}}{\mu_{\bx}\mu_{\blambda}}}+\frac{\zeta^{2}}{\mu_{\bx}\varepsilon}\Bigg)\log\left(\frac{\|\bx^{0} - \bx^{\star}\|^{2} + \| \blambda^{0} - \blambda^{\star}\|^{2}}{\varepsilon} \right)\right)$ 
    rounds, $\E\left[\mu_{\bx}\|\bx^{R} - \bx^{\star}\|^{2}\right] \leq \varepsilon$.
\end{itemize}
\end{corollary}

\begin{remark}
Our algorithm converges linearly to the global saddle point when each client applies a noiseless gradient for local updates (i.e., $\zeta=0$) in the presence of data heterogeneity and client-drift in federated learning. 
In contrast, previous approaches exhibit only sub-linear convergence. In the strongly-convex-strongly-concave setting, DRFA~\citep{deng2020distributionally} converges to the saddle-point solution with rate $O(1/R)$ when there is no data heterogeneity and $\zeta=0$.
\end{remark}

\begin{remark}
By applying bias correction in local updates, the convergence rates of our algorithm match those of the centralized accelerated primal-dual algorithm~\citep{zhang2021robust} in both deterministic and stochastic settings. 
\end{remark}

\begin{remark}
Compared to the standard minimization in federated learning, the DRO objective results in a slightly worse condition number in terms of convergence rate. 
In comparison to the standard minimization objective in federated learning, the DRO objective yields a slightly worse condition number. Solving DRO with {\algname} requires $(\sqrt{L_{\bx\bx}/\mu_{\bx}} + \sqrt{L_{\blambda\bx}^{2}/(L_{\bx\bx} \mu_{\blambda})})$ times more communication rounds compared to solving minimization problems with ProxSkip~\citep{mishchenko2022proxskip}.
\end{remark}

\section{Experiments}\label{sec:exp}

We now study the performance of {\algname} for solving  federated DRO problems on both synthetic datasets and real-world datasets. 
Our primary objective when working with synthetic datasets is to validate the convergence analysis of {\algname}. 
On real-world datasets, we compare with existing federated optimization algorithms for learning robust and fair models (DRFA~\citep{deng2020distributionally}, AFL~\citep{mohri2019agnostic}, and $q$-FFL~\citep{li2019fair}) as well as widely used federated algorithms for solving minimization problems including FedAvg~\citep{mcmahan2017communication} and SCAFFOLD~\citep{karimireddy2020scaffold}. 
After conducting thorough evaluations, we have observed that  our proposed accelerated algorithms achieve fast convergence rates and strong empirical performance on real-world datasets. 
We have provided supplementary experimental results in Appendix~\ref{sec:appendix-exp}, which includes additional baseline methods, ablations on our algorithm, and other relevant findings.

\vspace{-0.1in}
\subsection{Results on Synthetic Datasets}
To construct the synthetic datasets, we follow the setup described in Eq.~\eqref{eq:dro-def} and consider a simple robust regression problem. 
Specifically, for the $i$-th client, the local function $f_i$ is defined as $f_i(\bx) =  \frac{1}{m_i}\sum_{j=1}^{m_i}(\langle \ba_i^{j}, \bx\rangle - y_i^j)^{2} + \frac{\mu_{\bx}}{2}\|\bx\|^{2}$, where $j$ is sample index on this client and there are $m_i$ training samples on client-$i$. 
We apply the $\chi^{2}$ penalty for regularizing the weight vector $\blambda$. 
To generate the data, each input ${\ba_i^{j}}$ is sampled from a Gaussian distribution $\ba_i^{j} \sim \mathcal{N}(\bm{0}, \bm{I}_{d\times d})$. Then we random generate $\widehat{\bx} \sim \mathcal{N}(\bm{0}, c^{2}\bm{I}_{d\times d})$, and $\delta_{i}^{\bx} \sim \mathcal{N}(\bm{0}, \sigma^{2}\bm{I}_{d\times d})$. Based on $(\widehat{\bx}, \delta_{i}^{\bx})$,  we generate $y_i^{i}$ as $y_i^{i} = \langle \ba_i^j, \widehat{\bx} + \delta_{i}^{\bx}\rangle$. 
Therefore, there exist distribution shifts across different clients (i.e., concept shifts). 
We set $N=5$, $d=10$, and $m_i=100$ for $i\in[N]$. 
To measure the algorithm performance, we evaluate the distance between $\bx^{R}$ and the optimal solution $\bx^{\star}$: $\|\bx^{R} - \bx^{\star}\|^{2}$.

\vspace{0.05in}
We compare {\algname} with DRFA~\citep{deng2020distributionally} on this synthetic dataset. 
The regularization parameter $\rho$ for $\psi$ is varied from $0.01$ to $0.1$. 
For both algorithms, we set the number of local steps to be 100 and select the algorithm parameters through grid search. 
The comparison results are summarized in Fig~\ref{fig:exp-syn}. 
As shown in Fig~\ref{fig:exp-syn}, we observe that our proposed algorithm {\algname} achieves linear convergence rates in all three settings. 
In contrast, DRFA converges much more slowly compared to {\algname}. 
We have included more experimental results under this synthetic setup in Appendix~\ref{sec:appendix-exp}, including results on the effect of local steps and data heterogeneity.

\begin{figure}[t!]
     \centering
     \begin{subfigure}[b]{0.32\textwidth}
         \centering
    \includegraphics[width=\textwidth]{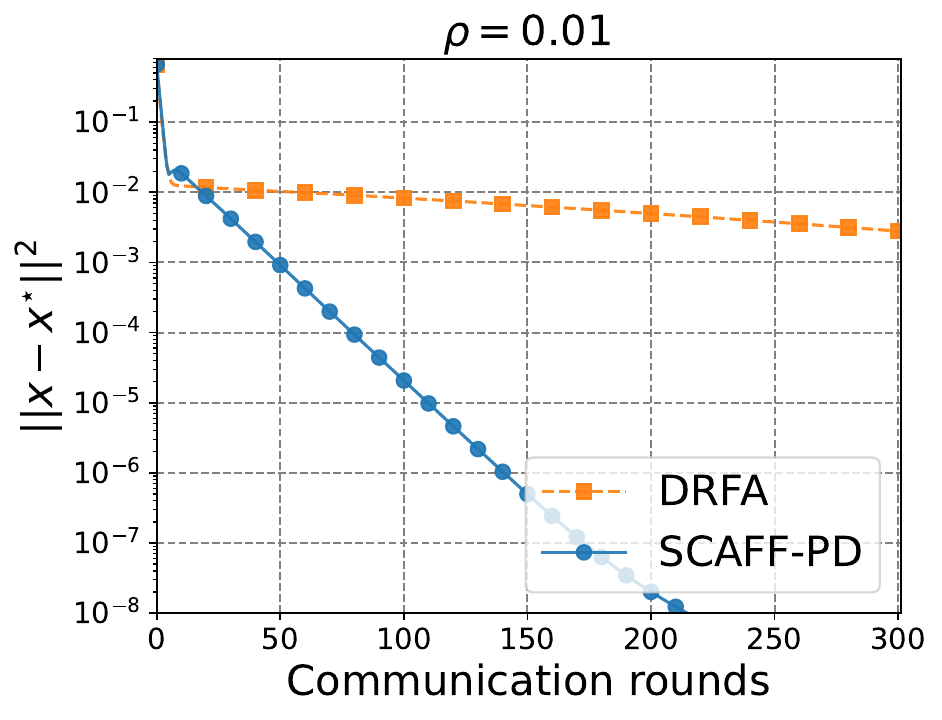}
         \caption{$\rho=0.01$.}
     \end{subfigure}
     \begin{subfigure}[b]{0.32\textwidth}
         \centering
    \includegraphics[width=\textwidth]{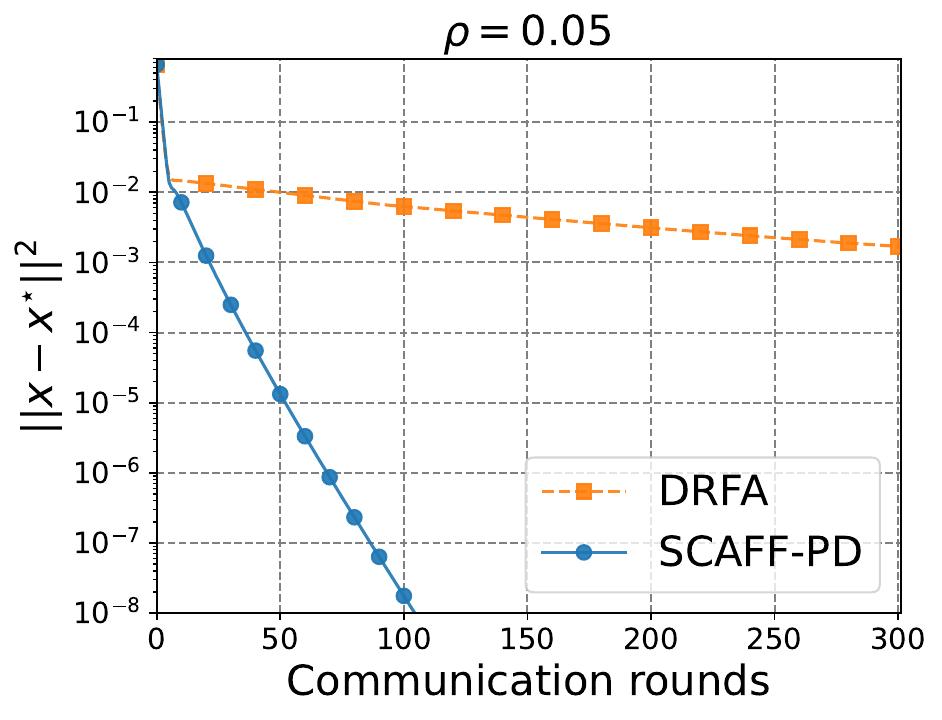}
         \caption{$\rho=0.05$.}
     \end{subfigure}
     \begin{subfigure}[b]{0.32\textwidth}
         \centering
    \includegraphics[width=\textwidth]{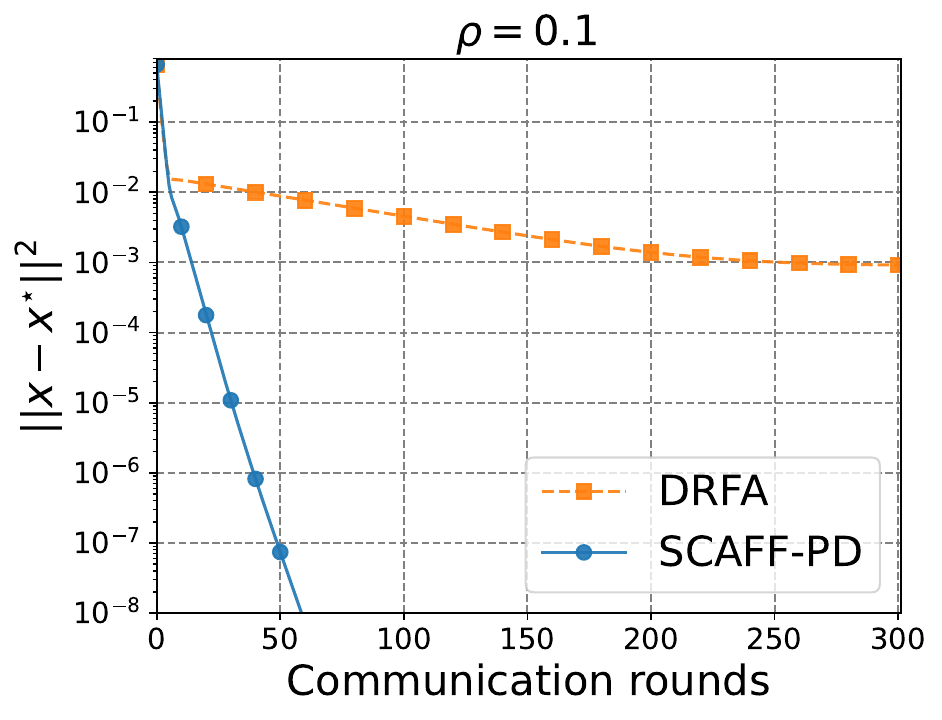}
         \caption{$\rho=0.1$.}
     \end{subfigure}
    \vspace{-0.05in}
        \caption{
        We compare our proposed algorithm with the existing method DFRA~\citep{deng2020distributionally} on synthetic datasets. $\rho$ is the strength of regularization $\psi$ (defined in Eq.~\eqref{eq:psi-function-def}). $X$-axis represents the number of communication rounds, and $Y$-axis represents the distance to optimal solution. 
        }
        \label{fig:exp-syn}
        \vspace{-0.2in}
\end{figure}

\subsection{Results on Real-world Datasets}

\textbf{Dataset setup.} 
We evaluate the performance of various federated learning algorithms on CIFAR100~\citep{krizhevsky2009learning} and TinyImageNet~\citep{le2015tiny}. 
We follow the setup used in \citet{li2022federated}: we consider different degrees of data heterogeneity by applying Dirichlet allocation, denoted by $\text{Dir}(\alpha)$, to partition the dataset into different clients. 
Smaller $\alpha$ values in $\text{Dir}(\alpha)$ leads to higher data heterogeneity. 
Additionally, after the data partition through the Dirichlet allocation, we randomly sample $30\%$ of the clients and remove $70\%$ training samples from those clients. 
Such a sub-sampling procedure can better model real-world data-imbalance scenarios.  
We consider the number of clients $N=20$ for both datasets. Results on larger number of clients and other real-world datasets can be found in Appendix~\ref{sec:appendix-exp}.

\vspace{0.05in}
\textbf{Model setup.} 
We consider learning a linear classifier by using representations extracted from pre-trained deep neural networks. 
Previous studies have demonstrated the efficacy of this approach, particularly in the context of data heterogeneity~\citep{yu2022tct} as well as sub-group robustness~\citep{izmailov2022feature}. 
For both datasets, we apply the ResNet-18~\citep{he2016deep} pre-trained on ImageNet-1k~\citep{deng2009imagenet} as the backbone for extracting feature representations of the image samples. 
To apply the pre-trained ResNet-18, we resize the images from CIFAR100 and TinyImageNet to 3$\times$224$\times$224.

\begin{table*}[t]
	\centering
	\caption{The average and worst-20\% top-1 accuracy of our algorithm ({\algname}) vs.\ state-of-the-art federated learning algorithms evaluated on CIFAR100 and Tiny-ImageNet. 
 The highest top-1 accuracy in each setting is highlighted in \textbf{bold}.
	}
	\label{tab:main_table}
	    \footnotesize
     \resizebox{0.95\textwidth}{!}{
		\begin{tabular}{cccc|cc|ccc}
			\toprule
			\multirow{1}{*}{\bf Datasets}  &   \multirow{1}{*}{\bf Methods}                 & \multicolumn{6}{c}{\bf Non-i.i.d. degree}   
			  \\
			  \midrule
			 &                   & \multicolumn{2}{c}{$\alpha=0.01$}    & \multicolumn{2}{c}{$\alpha=0.05$}    & \multicolumn{2}{c}{$\alpha=0.1$}   
			  \\
		\midrule
			\multirow{11}{*}{CIFAR-100}   &  &    
			\texttt{average} & \texttt{worst-20\%} & \texttt{average} & \texttt{worst-20\%} &\texttt{average} & \texttt{worst-20\%} & 
			\\ 
			\cmidrule(lr){1-9} 
			   & {FedAvg} & 38.77 & 15.93 & 35.96 & 24.43 & 36.57 & 26.50     \\ 
			\tablewhitespace
			   & {SCAFFOLD} & 37.38 & 14.65 & 35.28 & 24.77 & 35.63 & 25.61      \\
			  \tablewhitespace
			   & {$q$-FFL} & 26.39 & 5.43 & 29.60 & 18.62 & 30.38 & 21.98      \\
			  \tablewhitespace
			   & {AFL} & 47.38 & 18.04 & \textbf{44.73} & 22.06 & \textbf{44.89} & 27.27      \\
			  \tablewhitespace
			   & {DRFA} & 46.47 & 26.77 & 41.61 & 27.66 & 43.20 & 32.04      \\
			  \tablewhitespace
                 & \textit{SCAFF-PD}  & \textbf{49.03} & \textbf{29.30} & 42.06 & \textbf{28.37} & 43.69 & \textbf{32.77}     \\
                \midrule
   \multirow{11}{*}{TinyImageNet}   &  &    
			\texttt{average} & \texttt{worst-20\%} & \texttt{average} & \texttt{worst-20\%} &\texttt{average} & \texttt{worst-20\%} & 
			\\ 
			\cmidrule(lr){1-9} 
			   & {FedAvg} & 33.66 & 18.18 &  31.53 & 23.46   & 35.08 & 27.61     \\ 
			\tablewhitespace
			   & {SCAFFOLD} & 31.79 & 15.85  &  30.43 &  22.57 & 34.58 & 27.33      \\
			  \tablewhitespace
			   & {$q$-FFL} & 25.50 & 9.70 &  27.45 &  19.38  & 32.90 & 26.24      \\
			  \tablewhitespace
			   & {AFL} & \textbf{45.32} & 18.65 & \textbf{45.54}  & 28.02  & \textbf{46.11} & 29.50      \\
			  \tablewhitespace
			   & {DRFA} & 36.80 & 22.32 &  37.39  & 28.38 & 37.39 & 28.38      \\
			  \tablewhitespace
                 & \textit{SCAFF-PD} & 41.26 & \textbf{25.32}    & 39.32 & \textbf{30.27}  & 41.23 &  \textbf{29.78}  \\
			\bottomrule
		\end{tabular}
    }
	\vspace{-.75em}
\end{table*}

\vspace{0.05in}
\textbf{Comparisons with existing approaches.} 
We consider three data heterogeneity settings for both datasets. 
To measure the performance of different algorithms, beside the average classification accuracy across clients, we also evaluate the \texttt{worst-20\%} accuracy\footnote{First sort the clients by test  accuracy, then select the lower $20\%$ of clients and compute the mean from this subset.} for comparing fairness and robustness  of different federated learning algorithms. 
Previous studies have employed this metric for comparing different model in federated learning~\cite{li2019fair}.
The comparative results are summarized in Table~\ref{tab:main_table}. We find that our proposed algorithm outperforms existing methods in most settings, especially under higher heterogeneity. 
For example, when the level of data heterogeneity is low ($\alpha=0.1$), applying {\algname} does not yield very large improvements compared to the existing algorithms. 
In the case of high data heterogeneity ($\alpha=0.01$), our proposed algorithm largely improves the worst-20\% accuracy performance on both datasets.

\vspace{0.05in}
\textbf{Effect of $\rho$ in DRO.} 
To gain a better understanding of the empirical performance of our algorithm, we investigate the role of $\rho$ in DRO when applying our algorithm. We consider $\rho \in \{0.1, 0.2, 0.5\}$ and measure both the average and worst-20\% accuracy during training. 
We present the results in Fig~\ref{fig:exp-real-effect-rho}. 
We find that when $\rho$ is small, {\algname} can achieve better fairness/robustness---the worst-20\% accuracy significantly improves when we decrease the $\rho$ in {\algname}. 
Meanwhile, the experimental results suggest that smaller $\rho$ leads to faster convergence w.r.t. worst-20\% accuracy for our algorithm. 
On the other hand, when applying smaller $\rho$, the condition number of the min-max optimization problem becomes worse. 
Fortunately, our algorithm is guaranteed to  achieve accelerated rates, making it particularly beneficial in scenarios where $\mu_{\blambda}$ is small.
As we have demonstrated in Fig~\ref{fig:exp-syn}, our proposed algorithm still converges relatively fast when $\rho$ is small.

\vspace{0.05in}
In addition, we study the trade-off between average accuracy vs.\ worst-20\% accuracy vs.\ best-20\% accuracy for different algorithms. 
The results are summarized in Fig~\ref{fig:exp-real-scatter} (in Appendix~\ref{sec:appendix-exp}). 
Without sacrificing much on average accuracy and best-20\% accuracy, our algorithm largely improves the worst-20\% accuracy.

\begin{figure}[t!]
     \centering
     \begin{subfigure}[b]{0.49\textwidth}
         \centering
    \includegraphics[width=\textwidth]{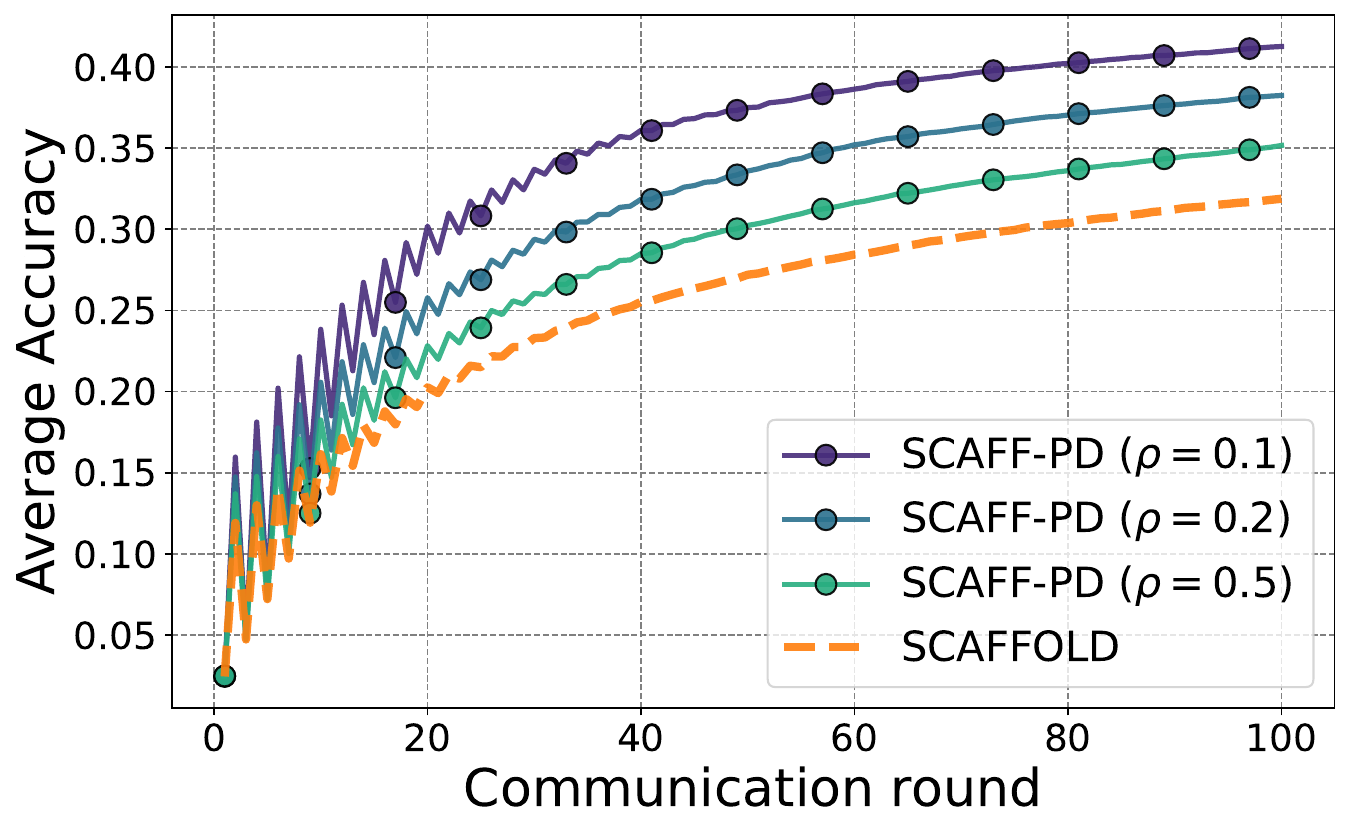}
         \caption{Average accuracy.}
     \end{subfigure}
     \begin{subfigure}[b]{0.49\textwidth}
         \centering
    \includegraphics[width=\textwidth]{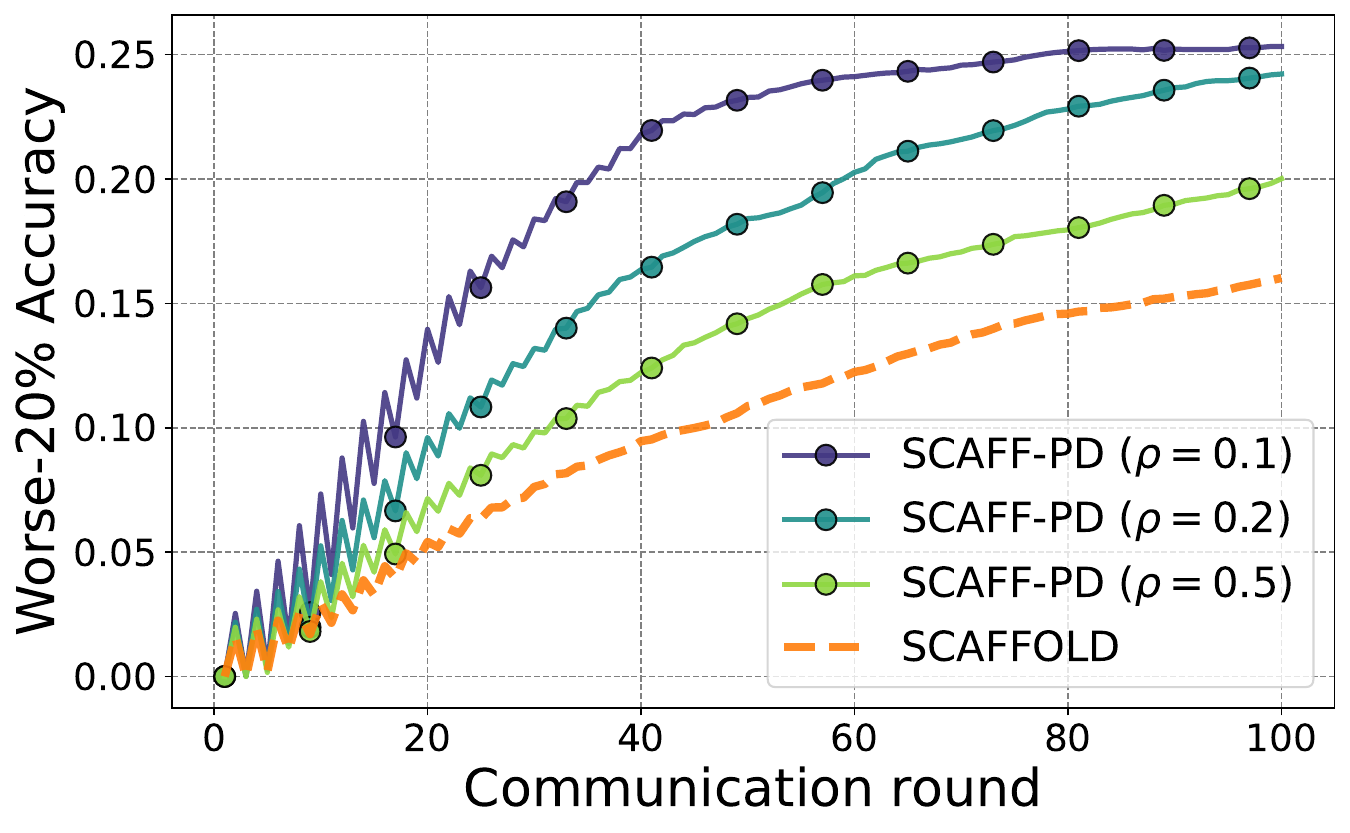}
         \caption{Worst-20\% accuracy.}
     \end{subfigure}
        \caption{
        We study the effect of regularization term $\rho$ in our proposed algorithm {\algname}. We measure both the average test accuracy (\textbf{a}) and worst-20\% accuracy (\textbf{b}) during training. 
        In addition, we include SCAFFOLD (orange dashed lines) as a baseline method for comparison.
        }
        \label{fig:exp-real-effect-rho}
        \vspace{-0.25in}
\end{figure}

\section{Conclusions}\label{eq:discussion}
We have demonstrated the ability of {\algname} to address challenges of fairness and robustness in federated learning. 
Theoretically, we obtained accelerated convergence rates for solving a wide class of federated DRO problems. 
Experimentally, we demonstrated strong empirical performance of {\algname} on real-world datasets, improving upon existing approaches in both communication efficiency and model performance.
An interesting future direction is the integration of DRO and privacy-preserving techniques in the context of federated learning, making {\algname} applicable for a wider range of real-world applications. Another exciting direction is to explicitly integrate {\algname} with game-theoretic mechanisms. Finally, studying the interplay between distributional robustness and personalization is an important open problem.


\bibliographystyle{plainnat}
\bibliography{ref}

\newpage
\appendix

\section{Technical Lemmas}
This section is dedicated to presenting several lemmas that serve as building blocks in proving the convergence of our proposed algorithms.

\begin{lemma}[Perturbed strong convexity, \citet{karimireddy2020scaffold}]\label{lemma:perturb-strong-convexity}
Suppose the function $f(\cdot): \mathcal{X} \rightarrow \mathbb{R}$ is $L$-smooth and $\mu$-strongly convex, then for any $\bx, \by, \bz \in \mathcal{X}$,
\begin{equation}
    \langle \nabla f(\bx), \bz -\by \rangle \geq f(\bz) - f(\by) + \frac{\mu}{4}\|\by - \bz\|^{2} - L\|\bz - \bz\|^{2}.
\end{equation}
\end{lemma}

We now present the lemma for analyzing the drift term.
\begin{lemma}[Bounded drift]\label{Lemma:drift-lemma-deterministic}
Suppose $\tau_r = J\,\eta_{\ell}\,\eta_{g}$, and $\eta_{g}\geq 1$, then we have
\begin{equation}
    \mathcal{E}_{r} \leq \frac{12\tau^{2}}{\eta_{g}^{2}}\E\left[\left\|\nabla_{\bx} \Phi(\bx^{r}, \blambda^{r+1})\right\|^{2}\right] + \frac{12\tau^{2}}{\eta_{g}^{2}}\left(1+\chi\right)\zeta^{2} + \frac{3\tau^{2}}{\eta_g^{2} J}\zeta^{2},
\end{equation}
where $\mathcal{E}_r$ is defined as
\begin{equation}
\mathcal{E}_{r} = \frac{1}{J}\sum_{i=1}^{N}\sum_{j=1}^{J}\lambda_{i}\E\left[\|\bu_{i, j} - \bx\|^{2}\right],
\end{equation}
and $\chi$ is defined as
\begin{equation}
    \chi = \max_{\blambda \in \Lambda} \sum_{i=1}^{N}\lambda_i^{2}.
\end{equation}
\end{lemma}

\begin{proof}
We omit the $r$ superscript in the following proof. 
Recall that the definition of $\bu_{i,j}$ in Algorithm~\ref{Algorithm:ScaffoldAPD}, i.e.,
$$\bu_{i,j}=\bu_{i,j-1} - \eta_{\ell}\left(g_i(\bu_{i,j-1}) - \hat{\bc}_i + \hat{\bc}\right),$$ 
and we have
\begin{equation}
\begin{aligned}
    \E\left[g_i(\bu_{i,j-1})\right] &= \nabla f_i(\bu_{i,j-1}),\\
    \E\left[\hat{\bc}_i\right] &= \nabla f_i(\bx) = \bc_i,\\
    \E\left[\hat{\bc}\right] &=  
    \sum_{i=1}^{N}\lambda_i \nabla f_i(\bx) = \bc.
\end{aligned}
\end{equation}
Then we can upper bound $\E\left[\|\bu_{i, j} - \bx\|^{2}\right]$ as follows,
\begin{equation}
\begin{aligned}
&\E\left[\|\bu_{i, j} - \bx\|^{2}\right] \\
=\,& \E\left[\|\bu_{i, j-1} - \bx - \eta_{\ell}( g_i(\bu_{i,j-1})- \hat{\bc_i} + \hat{\bc})\|^{2}\right]\\
=\,& \E\left[\|\bu_{i, j-1} - \bx - \eta_{\ell}( \nabla f_i(\bu_{i,j-1})- \hat{\bc_i} + \hat{\bc})\|^{2}\right] + \eta_{\ell}^{2}\E\left[\| g_i(\bu_{i,j-1}) - \nabla f_i(\bu_{i,j-1})\|^{2}\right]\\
\leq\,& \E\left[\|\bu_{i, j-1} - \bx - \eta_{\ell}( \nabla f_i(\bu_{i,j-1})- \hat{\bc_i} + \hat{\bc})\|^{2}\right] + \eta_{\ell}^{2}\zeta^{2}\\
\leq\,& \left(1 + \frac{1}{J-1}\right)\E\left[\|\bu_{i, j-1} - \bx \|^{2}\right] + J\eta_{\ell}^{2}\E\left[\|\nabla f_i(\bu_{i,j-1})- \hat{\bc}_i + \hat{\bc}\|^{2}\right] + \eta_{\ell}^{2}\zeta^{2}\\
=\,& \left(1 + \frac{1}{J-1}\right)\E\left[\|\bu_{i, j-1} - \bx \|^{2}\right] + \frac{\tau^{2}}{\eta_{g}^{2}J}\E\left[\|\nabla f_i(\bu_{i,j-1})- \hat{\bc}_i + \hat{\bc}\|^{2}\right] + \eta_{\ell}^{2}\zeta^{2} \\
\leq\,& \left(1 + \frac{1}{J-1}\right)\E\left[\|\bu_{i, j-1} - \bx \|^{2}\right] + \frac{2\tau^{2}}{\eta_{g}^{2}J}\E\left[\|\nabla f_i(\bu_{i,j-1})- {\bc}_i + {\bc}\|^{2}\right]  \\
& \quad + \frac{2\tau^{2}}{\eta_{g}^{2}J}\E\left[\|\hat{\bc}_i -  {\bc}_i + {\bc} - \hat{\bc}\|^{2}\right]+ \eta_{\ell}^{2}\zeta^{2} \\
\leq\,& \left(1 + \frac{1}{J-1}\right)\E\left[\|\bu_{i, j-1} - \bx \|^{2}\right] + \frac{2\tau^{2}}{\eta_{g}^{2}J}\E\left[\|\nabla f_i(\bu_{i,j-1})- {\bc}_i + {\bc}\|^{2}\right]  \\
& \quad + 
\underbrace{\frac{4\tau^{2}}{\eta_{g}^{2}J}\E\left[\|\hat{\bc}_i -  {\bc}_i\|^{2}\right] + \frac{4\tau^{2}}{\eta_{g}^{2}J}\E\left[\| \hat{\bc} - {\bc}\|^{2}\right] + \eta_{\ell}^{2}\zeta^{2}}_{ \Gamma},
\end{aligned}
\end{equation}
where $\Gamma = 0$ if the local gradients are deterministic. Next, we could first upper bound the term $\E\left[\|\bu_{i, j} - \bx\|^{2}\right]$ as 
\begin{equation}
\begin{aligned}
&\E\left[\|\bu_{i, j} - \bx\|^{2}\right] \\
\leq\,& \left(1 + \frac{1}{J-1}\right)\E\left[\|\bu_{i, j-1} - \bx \|^{2}\right] + \frac{4\tau^{2}}{\eta_{g}^{2}J}\E\left[\|\nabla f_i(\bu_{i,j-1})- \bc_i\|^{2}\right]+ \frac{4\tau^{2}}{\eta_{g}^{2}J}\E\left[\| \bc\|^{2}\right] + \Gamma\\
\leq\,& \left(1 + \frac{1}{J-1}+\frac{4\tau^{2}L_{\bx\bx}^{2}}{\eta_{g}^{2}J}\right)\E\left[\|\bu_{i, j-1} - \bx \|^{2}\right] + \frac{4\tau^{2}}{\eta_{g}^{2}J}\E\left[\| \bc\|^{2}\right] + \Gamma \\
\leq\,& \left(1 + \frac{2}{J-1}\right)\E\left[\|\bu_{i, j-1} - \bx \|^{2}\right]  + \frac{4\tau^{2}}{\eta_{g}^{2}J}\E\left[\| \bc\|^{2}\right] + \Gamma ,
\end{aligned}
\end{equation}
where we apply the condition that $\frac{4\tau^{2}L_{\bx\bx}^{2}}{\eta_{g}^{2}J} \leq \frac{1}{J-1}$. Then we have
\begin{equation}
\begin{aligned}
    \E\left[\|\bu_{i, j} - \bx\|^{2}\right] 
    &\leq \sum_{i=1}^{j-1} \left(1 + \frac{2}{J-1}\right)^{i} \left(\frac{4\tau^{2}}{\eta_{g}^{2}J}\E\left[\| \bc\|^{2}\right] + \Gamma \right) \\
    &\leq 3J \left(\frac{4\tau^{2}}{\eta_{g}^{2}J}\E\left[\| \bc\|^{2}\right] + \Gamma \right)
    = \frac{12\tau^{2}}{\eta_{g}^{2}}\E\left[\| \bc\|^{2}\right] + 3J\,\Gamma.
\end{aligned}
\end{equation}
Then the drift error $\mathcal{E}_{r}$ can be upper bounded  as follows,
\begin{equation*}
\begin{aligned}
    \mathcal{E}_{r} & = \frac{1}{J}\sum_{i=1}^{N}\sum_{j=1}^{J}\lambda_{i}\E\left[\|\bu_{i, j} - \bx\|^{2}\right] \\
    &\leq \frac{1}{J}\sum_{i=1}^{N}\sum_{j=1}^{J}\lambda_i \frac{12\tau^{2}}{\eta_{g}^{2}}\E\left[\| \bc\|^{2}\right] + \frac{1}{J}\sum_{i=1}^{N}\sum_{j=1}^{J}3\lambda_i J\,\Gamma \\
    &=  \frac{12\tau^{2}}{\eta_{g}^{2}}\E\left[\Big\|\sum_{i=1}^{N}\lambda_{i}\nabla f_i(\bx)\Big\|^{2}\right] + 3J\cdot\left(\frac{4\tau^{2}}{\eta_{g}^{2}J}\sum_{i=1}^{N}\lambda_i\E\left[\|\hat{\bc}_i -  {\bc}_i\|^{2}\right] + \frac{4\tau^{2}}{\eta_{g}^{2}J}\E\left[\| \hat{\bc} - {\bc}\|^{2}\right] + \eta_{\ell}^{2}\zeta^{2}\right) \\
    &\leq  \frac{12\tau^{2}}{\eta_{g}^{2}}\E\left[\Big\|\sum_{i=1}^{N}\lambda_{i}\nabla f_i(\bx)\Big\|^{2}\right] + \left(\frac{12\tau^{2}}{\eta_{g}^{2}}\sum_{i=1}^{N}\lambda_i \E\left[\|\hat{\bc}_i -  {\bc}_i\|^{2}\right] + \frac{12\tau^{2}}{\eta_{g}^{2}}\E\left[\| \hat{\bc} - {\bc}\|^{2}\right] + \frac{3\tau^{2}}{\eta_g^{2} J}\zeta^{2}\right) \\
    &\leq  \frac{12\tau^{2}}{\eta_{g}^{2}}\E\left[\left\|\nabla_{\bx} \Phi(\bx, \blambda)\right\|^{2}\right] + \frac{12\tau^{2}}{\eta_{g}^{2}}\zeta^{2} +  \frac{12\tau^{2}}{\eta_{g}^{2}}\underbrace{\left(\sum_{i=1}^{N}\lambda_i^{2} \right)}_{\leq \chi}\zeta^{2} + \frac{3\tau^{2}}{\eta_g^{2} J}\zeta^{2}\\
    &\leq  \frac{12\tau^{2}}{\eta_{g}^{2}}\E\left[\left\|\nabla_{\bx} \Phi(\bx, \blambda)\right\|^{2}\right] + \frac{12\tau^{2}}{\eta_{g}^{2}}\left(1+\chi\right)\zeta^{2} + \frac{3\tau^{2}}{\eta_g^{2} J}\zeta^{2},
\end{aligned}
\end{equation*}
which completes the proof.
\end{proof}

The next lemma is useful in effectively controlling the drift term in our later analysis.

\begin{lemma}\label{Lemma:x-update-lowerbound-drift}
Suppose $\tau_r = J\,\eta_{\ell}\,\eta_{g}$, and $\eta_{g}\geq 1$, then we have
\begin{equation}
\frac{1}{\tau_{r}}\E\left[\|\bx^{r+1}-\bx^{r}\|^{2}\right] \geq -{\tau_r L_{\bx\bx}^{2}} {\mathcal{E}_r} + \frac{\tau_r}{2}\left\|\nabla_{\bx}\Phi(\bx^{r}, \blambda^{r+1})\right\|^{2},
\end{equation}
where $\nabla_{\bx}\Phi(\bx^{r}, \blambda^{r+1})$ is defined as
\begin{equation}
    \nabla_{\bx}\Phi(\bx^{r}, \blambda^{r+1}) = \sum_{i=1}^{N}\lambda_i^{r+1}\nabla f_i(\bx^{r}).
\end{equation}
\end{lemma}

\begin{proof}
In start with, we analyze $\frac{1}{4\tau_{r}}\|\bx^{r+1}-\bx^{r}\|^{2} $   based on the local updates, i.e.,
\begin{equation*}
\begin{aligned}
&\quad\,\,\frac{1}{\tau_{r}}\E\left[\|\bx^{r+1}-\bx^{r}\|^{2}\right]\\
&= {\tau_r}\E\left[\Big\|\frac{1}{J}\sum_{i=1}^{N}\sum_{j=1}^{J}\lambda_i^{r+1}g_i(\bu^{r}_{i, j-1})\Big\|^{2}\right] \\
& \geq {\tau_r}\E\left[\Big\|\frac{1}{J}\sum_{i=1}^{N}\sum_{j=1}^{J}\lambda_i^{r+1}\nabla f_i(\bu^{r}_{i, j-1})\Big\|^{2}\right]  \\
&\geq -{\tau_r}\E\left[\Big\|\frac{1}{J}\sum_{i=1}^{N}\sum_{j=1}^{J}\lambda_i^{r+1}\nabla f_i(\bu^{r}_{i, j-1}) - \sum_{i=1}^{N}\lambda_i^{r+1}\nabla f_i(\bx^{r})\Big\|^{2}\right] + \frac{\tau_r}{2}\E\left[\left\|\nabla_{\bx}\Phi(\bx^{r}, \blambda^{r+1})\right\|^{2}\right] \\
&\geq -{\tau_r}\frac{1}{J}\sum_{i=1}^{N}\sum_{j=1}^{J}\lambda_i^{r+1}\E\left[\left\|\nabla f_i(\bu^{r}_{i, j-1}) - \nabla f_i(\bx^{r})\right\|^{2}\right] + \frac{\tau_r}{2}\E\left[\left\|\nabla_{\bx}\Phi(\bx^{r}, \blambda^{r+1})\right\|^{2}\right] \\
&\geq -{\tau_r L_{\bx\bx}^{2}}\underbrace{\frac{1}{J}\sum_{i=1}^{N}\sum_{j=1}^{J}\lambda_i^{r+1}\E\left[\left\|\bu^{r}_{i, j-1} - \bx^{r} \right\|^{2}\right]}_{\mathcal{E}_r} + \frac{\tau_r}{2}\E\left[\left\|\nabla_{\bx}\Phi(\bx^{r}, \blambda^{r+1})\right\|^{2}\right]\\
&= -{\tau_r L_{\bx\bx}^{2}} {\mathcal{E}_r} + \frac{\tau_r}{2}\E\left[\left\|\nabla_{\bx}\Phi(\bx^{r}, \blambda^{r+1})\right\|^{2}\right],
\end{aligned}
\end{equation*}
which completes the proof.
\end{proof}

The next two lemmas focus on the primal update and dual update, respectively.
\begin{lemma}\label{Lemma:lamba-update-optimal-confidtion}
Suppose $\tau_r = J\,\eta_{\ell}\,\eta_{g}$, and $\eta_{g}\geq 1$, then we have
\begin{equation}\label{eq:useful-lemma-apd-lambda-update}
\begin{aligned}
&\psi(\blambda^{r+1}) - \langle{\bs}^{r}, \blambda^{r+1}\rangle \\
\leq &\,\psi(\blambda) - \langle{\bs}^{r}, \blambda\rangle + \frac{1}{\sigma_r}\left[\mathrm{D}(\blambda, \blambda^{r})-\mathrm{D}(\blambda, \blambda^{r+1})-\mathrm{D}(\blambda^{r+1}, \blambda^{r})\right] - \frac{\mu_{\blambda}}{2}\|\blambda_{r+1} - \blambda\|^{2}.
\end{aligned}
\end{equation}
\end{lemma}
\begin{proof}
Based on Property 1 in \citet{tseng2008accelerated}, and $\mathrm{D}(\blambda, \blambda^{\prime})= \|\blambda - \blambda^{\prime}\|^{2}/2$.
\end{proof}

\begin{lemma}\label{Lemma:x-update-optimal-confidtion}
Suppose $\tau_r = J\,\eta_{\ell}\,\eta_{g}$, and $\eta_{g}\geq 1$, then we have
\begin{equation}\label{eq:useful-lemma-apd-x-update}
    \E\left[\Big\langle \Delta\bx^{r}, \bx^{r+1} - \bx\Big\rangle\right] \leq  \frac{1}{\tau_r}\E\left[\mathrm{D}(\bx, \bx^{r})-\mathrm{D}(\bx, \bx^{r+1})-\mathrm{D}(\bx^{r+1}, \bx^{r})\right],
\end{equation}
and 
\begin{equation}
\begin{aligned}
&\E\left[\Big\langle \Delta{\bx}^{r}, \bx^{r+1} - \bx\Big\rangle\right]\\
=\,& \underbrace{\E\left[\Big\langle \Delta{\bx}^{r}, \bx^{r} - \bx\Big\rangle\right]}_{\cT_1} +  
\underbrace{\E\left[\Big\langle \widetilde{\Delta}{\bx}^{r}, \bx^{r+1} - \bx^{r}\Big\rangle\right]}_{\cT_2} + \underbrace{\E\left[\Big\langle \Delta{\bx}^{r} - \widetilde{\Delta}{\bx}^{r}, \bx^{r+1} - \bx^{r}\Big\rangle\right]}_{\cT_3},
\end{aligned}
\end{equation}
where
\begin{equation}
\begin{aligned}
    \cT_1 \geq \,& \E\left[{\Phi(\bx^{r}, \blambda^{r+1})} - {\Phi(\bx, \blambda^{r+1})}\right] + \frac{\mu_{\bx}}{4}\E\left[\|\bx^{r}-\bx\|^{2}\right] - L_{\bx\bx}{\mathcal{E}_{r}},\\
    \cT_2 \geq \,& \E\left[{\Phi(\bx^{r+1}, \blambda^{r+1})} - {\Phi(\bx^{r}, \blambda^{r+1})}\right]  - 2L_{\bx\bx}\E\left[\|\bx^{r+1}-\bx^{r}\|^{2}\right] - 2L_{\bx\bx}\mathcal{E}_{r},\\
    \cT_3 \geq \, &-\frac{2\chi\,\tau_r}{J}\zeta^{2} - \frac{1}{4\tau_r}\E\left[\mathrm{D}(  \bx^{r+1}, \bx^{r})\right],
\end{aligned}
\end{equation}
and ${\Delta}{\bx}^{r}$ and $\widetilde{\Delta}{\bx}^{r}$ are defined as
\begin{equation}\label{eq:appendix-delta-x-r-def}
    {\Delta}{\bx}^{r} = \frac{1}{J}\sum_{i=1}^{N}\sum_{j=1}^{J}\lambda_i^{r+1} g_i(\bu^{r}_{i, j-1}), \,\,\,
    \widetilde{\Delta}{\bx}^{r} = \E\left[\Delta\bx^{r}\right] = \frac{1}{J}\sum_{i=1}^{N}\sum_{j=1}^{J}\lambda_i^{r+1}\E\left[\nabla f_i(\bu^{r}_{i, j-1})\right].
\end{equation}
\end{lemma}

\begin{proof}
Based on Property 1 in \citet{tseng2008accelerated}, for the update step of $\bx^{r+1}$, we have
\begin{equation}
    \E\left[\Big\langle \sum_{i=1}^{N}\lambda_i^{r+1}\Delta\bu^{r}_{i}, \bx^{r+1} - \bx\Big\rangle\right] \leq  \frac{1}{\tau_r}\E\left[\mathrm{D}(\bx, \bx^{r})-\mathrm{D}(\bx, \bx^{r+1})-\mathrm{D}(\bx^{r+1}, \bx^{r})\right],
\end{equation}
then we analyze the term $\Delta{\bx}^{r} = \sum_{i=1}^{N}\lambda_i^{r+1}\Delta\bu^{r}_{i}$, i.e.,
\begin{equation}
\begin{aligned}
\sum_{i=1}^{N}\lambda_i^{r+1}\Delta\bu^{r}_{i}
&= \frac{1}{J}\sum_{i=1}^{N}\sum_{j=1}^{J}\lambda_i^{r+1}\left(g_i(\bu^{r}_{i, j-1}) - \hat{\bc}_i^{r} + \hat{\bc}^{r}\right) \\
&= \frac{1}{J}\sum_{i=1}^{N}\sum_{j=1}^{J}\lambda_i^{r+1}g_i(\bu^{r}_{i, j-1})  -  \frac{1}{J}\sum_{i=1}^{N}\sum_{j=1}^{J}\lambda_i^{r+1}\hat{\bc}_i^{r} +  \frac{1}{J}\sum_{i=1}^{N}\sum_{j=1}^{J}\lambda_i^{r+1} \hat{\bc}^{r} \\
&= \frac{1}{J}\sum_{i=1}^{N}\sum_{j=1}^{J}\lambda_i^{r+1}g_i(\bu^{r}_{i, j-1}).
\end{aligned}
\end{equation}
Next we decompose $\E\left[\Big\langle \Delta{\bx}^{r}, \bx^{r+1} - \bx\Big\rangle\right]$ as follows,
\begin{equation}
\begin{aligned}
&\E\left[\Big\langle \Delta{\bx}^{r}, \bx^{r+1} - \bx\Big\rangle\right]\\
=\,& \underbrace{\E\left[\Big\langle \Delta{\bx}^{r}, \bx^{r} - \bx\Big\rangle\right]}_{\cT_1} +  
\underbrace{\E\left[\Big\langle \widetilde{\Delta}{\bx}^{r}, \bx^{r+1} - \bx^{r}\Big\rangle\right]}_{\cT_2} + \underbrace{\E\left[\Big\langle \Delta{\bx}^{r} - \widetilde{\Delta}{\bx}^{r}, \bx^{r+1} - \bx^{r}\Big\rangle\right]}_{\cT_3}.
\end{aligned}
\end{equation}
We then analyze the upper bound for $|\cT_3|$, i.e.,
\begin{equation}
\begin{aligned}
|\cT_3| &= \E\left[\big|\langle \Delta{\bx}^{r} - \widetilde{\Delta}{\bx}^{r}, \bx^{r+1} - \bx^{r}\rangle\big|\right] \\
&\leq {\tau_r} \E\left[\| \Delta{\bx}^{r} - \widetilde{\Delta}{\bx}^{r}\|^{2}\right] + \frac{1}{4\tau_r}\E\left[\|  \bx^{r+1} - \bx^{r}\|^{2}\right]\\
&\leq \frac{2\chi\,\tau_r}{J} \zeta^{2} + \frac{1}{4\tau_r}\E\left[\|  \bx^{r+1} - \bx^{r}\|^{2}\right].
\end{aligned}
\end{equation}
Next, we analyze term $\cT_1$, i.e.,
\begin{equation*}
\begin{aligned}
\cT_1=\,&\E\left[\Big\langle \Delta{\bx}^{r}, \bx^{r} - \bx\Big\rangle\right] = \E\left[\Big\langle \widetilde{\Delta}{\bx}^{r}, \bx^{r} - \bx\Big\rangle\right] \\
=\,& \E\Big[\Big\langle \frac{1}{J}\sum_{i=1}^{N}\sum_{j=1}^{J}\lambda_i^{r+1}\nabla f_i(\bu^{r}_{i, j-1}), \bx^{r} - \bx\Big\rangle\Big]\\
\geq& \frac{1}{J}\sum_{i, j} \lambda_i^{r+1}\E\left[f_i(\bx^{r}) - f_i(\bx)+\frac{\mu_{\bx}}{4}\|\bx^{r}-\bx\|^{2}-L_{\bx\bx}\|\bu_{i,j-1}^{r} -\bx^{r}\|^{2}\right] \\
=& \E\left[\sum_{i} \lambda_i^{r+1}f_i(\bx^{r}) - \sum_{i} \lambda_i^{r+1}f_i(\bx) + \frac{\mu_{\bx}}{4}\|\bx^{r}-\bx\|^{2} - L_{\bx\bx}\frac{1}{J}\sum_{i,j} \lambda_i^{r+1}\|\bu_{i,j-1}^{r} -\bx^{r}\|^{2}\right]  \\
=& \E\left[{\Phi(\bx^{r}, \blambda^{r+1})} - {\Phi(\bx, \blambda^{r+1})}\right] + \frac{\mu_{\bx}}{4}\E\left[\|\bx^{r}-\bx\|^{2}\right] - L_{\bx\bx}{\mathcal{E}_{r}},
\end{aligned}
\end{equation*}
where we apply the perturbed strong convexity lemma (Lemma~\ref{lemma:perturb-strong-convexity}) for the first inequality. 

We then analyze term $\cT_2$, 
\begin{equation*}
\begin{aligned}
\cT_2 =\,& \E\left[\Big\langle \widetilde{\Delta}{\bx}^{r}, \bx^{r+1} - \bx^{r}\Big\rangle\right] \\
=\,& \E\Big[\Big\langle \frac{1}{J}\sum_{i=1}^{N}\sum_{j=1}^{J}\lambda_i^{r+1}\nabla f_i(\bu^{r}_{i, j-1}), \bx^{r+1} - \bx^{r}\Big\rangle\Big]\\
\geq\,& \E\left[{\Phi(\bx^{r+1}, \blambda^{r+1})} - {\Phi(\bx^{r}, \blambda^{r+1})} + \frac{\mu_{\bx}}{4}\|\bx^{r+1}-\bx^{r}\|^{2} - {\frac{L_{\bx\bx}}{J}\sum_{i,j} \lambda_i^{r+1}\|\bu_{i,j-1}^{r} -\bx^{r+1}\|^{2}}\right]\\
\geq\,& \E\left[{\Phi(\bx^{r+1}, \blambda^{r+1})} - {\Phi(\bx^{r}, \blambda^{r+1})} + \left(\frac{\mu_{\bx}}{4} - 2L_{\bx\bx}\right)\|\bx^{r+1}-\bx^{r}\|^{2}\right]   - 2L_{\bx\bx}\mathcal{E}_r \\
\geq\,& \E\left[{\Phi(\bx^{r+1}, \blambda^{r+1})} - {\Phi(\bx^{r}, \blambda^{r+1})}\right]  - 2L_{\bx\bx}\E\left[\|\bx^{r+1}-\bx^{r}\|^{2}\right] - 2L_{\bx\bx}\mathcal{E}_{r},
\end{aligned}
\end{equation*}
where we apply the perturbed strong convexity lemma (Lemma~\ref{lemma:perturb-strong-convexity}) for the first inequality and apply $\|\bx + \by\|^{2}\leq 2\|\bx\|^{2} + 2\|\by\|^{2}$ for the second inequality.  This completes our proof.
\end{proof}

\newpage

\section{Convergence of \algname}
In this section, we present the missing proofs in Section~\ref{sec:theory}. 
Specifically, Section~\ref{subsec:appendix-proof-sc-c} contains the proofs for the strongly convex-concave setting (Section~\ref{sec:theory-sc-c}), while Section~\ref{subsec:appendix-proof-sc-sc} includes the proofs for the strongly convex-strongly concave setting (Section~\ref{sec:theory-sc-sc}).

\subsection{Proofs -- strongly-convex-concave (SC-C) setting}\label{subsec:appendix-proof-sc-c}
In this subsection, we first present the technical lemma in Section~\ref{subsec:appendix-proof-sc-c-lemmas}. 
Next, we analyze how to set the step size related parameters in Section~\ref{subsec:appendix-proof-sc-c-parameters}.
Finally, we prove Theorem~\ref{thm:main-sc-c} in Section~\ref{subsec:appendix-proof-sc-c-theorem-proof}.

\subsubsection{Technical Lemma}\label{subsec:appendix-proof-sc-c-lemmas}

\begin{lemma}\label{lemma:sc-c-stochastic}
If we set the step size in Algorithm~\ref{Algorithm:ScaffoldAPD} as $\tau_r\cdot L_{\bx\bx} \leq 1$, and the parameters of Algorithm~\ref{Algorithm:ScaffoldAPD} satisfy Condition~\ref{condition:stepsize-scc}, 
then for any $\bx, \blambda$ we have
\begin{equation}\label{eq:lemma:sc-c-stochastic-eq0}
\begin{aligned}
\E\left[{F(\bx^{r+1}, \blambda)} -  {F(\bx, \blambda^{r+1})}\right] \leq -Z_{r+1} + V_r + \Delta_r+ C\tau_r \zeta^{2},
\end{aligned}
\end{equation}
where $Z_{r+1}, V_r, \Delta_r$ are defined as
\begin{small}
\begin{align}
Z_{r+1} &=\E\left[ \langle \bq^{r+1}, \blambda^{r+1} - \blambda \rangle + \frac{1}{2\sigma_r}\|\blambda^{r+1} - \blambda\|^{2} + \left(\frac{1}{2\tau_r}+\frac{\mu_{\bx}}{8}\right)\|\bx^{r+1} -\bx\|^{2} + \frac{1}{2\alpha_{r+1}}\|\bq^{r+1}\|^{2}\right],\nonumber \\
V_{r} &=  \E\left[ \theta_{r}\langle\bq^{r}, \blambda^{r} - \blambda\rangle +  \frac{1}{2\sigma_r}\|\blambda^{r} - \blambda\|^{2} + \frac{1}{2\tau_r}\|\bx^{r} - \bx\|^{2} + \frac{\theta_{r}}{2\alpha_r}\|\bq^{r}\|^{2}\right], \\
\Delta_r &= \E\left[ \left(\frac{\alpha_{r}\theta_{r}}{2} - \frac{1}{2\sigma_r}\right)\|\blambda^{r+1}-\blambda^{r}\|^{2} + \left(\frac{L_{\blambda\bx}^{2}}{2\alpha_{r+1}} + 3L_{\bx\bx} - \frac{1}{4\tau_{r}}\right)\|\bx^{r+1}-\bx^{r}\|^{2}\right],\nonumber
\end{align}
\end{small}

$\bq^{r}$ is defined as
\begin{equation}
    \bq^r = \nabla_{\blambda}\Phi(\bx^{r}, \blambda^{r}) - \nabla_{\blambda}\Phi(\bx^{r-1}, \blambda^{r-1}),
\end{equation}
and $C \geq 0$ is a constant.
\end{lemma}
\begin{proof}
To start with, by applying Lemma~\ref{Lemma:lamba-update-optimal-confidtion}, we have
\begin{equation}\label{eq:apd-y-stoc}
    \psi(\blambda^{r+1}) - \langle\bs^{r}, \blambda^{r+1} - \blambda\rangle \leq \psi(\blambda) + \underbrace{\frac{1}{\sigma_r}\left[\mathrm{D}(\blambda, \blambda^{r})-\mathrm{D}(\blambda, \blambda^{r+1})-\mathrm{D}(\blambda^{r+1}, \blambda^{r})\right]}_{B_r},
\end{equation}
where we define $B_r$ as
\begin{equation}
    B_r = \frac{1}{\sigma_r}\left[\mathrm{D}(\blambda, \blambda^{r})-\mathrm{D}(\blambda, \blambda^{r+1})-\mathrm{D}(\blambda^{r+1}, \blambda^{r})\right].
\end{equation}
Then by applying Lemma~\ref{Lemma:x-update-optimal-confidtion}, for the update step of $\bx^{r+1}$, we have
\begin{equation}\label{eq:apd-1-stoc}
    \E\left[\left\langle \Delta\bx^{r}, \bx^{r+1} - \bx\right\rangle\right] \leq  \frac{1}{\tau_r}\E\left[\mathrm{D}(\bx, \bx^{r})-\mathrm{D}(\bx, \bx^{r+1})-\mathrm{D}(\bx^{r+1}, \bx^{r})\right],
\end{equation}
where $\Delta\bx^{r}$ is defined in Eq.~\eqref{eq:appendix-delta-x-r-def}. Then we decompose the term $\E[\left\langle \Delta\bx^{r}, \bx^{r+1} - \bx\right\rangle]$ as follows,
\begin{equation}
\begin{aligned}
&\E\left[\Big\langle \Delta{\bx}^{r}, \bx^{r+1} - \bx\Big\rangle\right]\\
=\,& \underbrace{\E\left[\Big\langle \Delta{\bx}^{r}, \bx^{r} - \bx\Big\rangle\right]}_{\cT_1} +  
\underbrace{\E\left[\Big\langle \widetilde{\Delta}{\bx}^{r}, \bx^{r+1} - \bx^{r}\Big\rangle\right]}_{\cT_2} + \underbrace{\E\left[\Big\langle \Delta{\bx}^{r} - \widetilde{\Delta}{\bx}^{r}, \bx^{r+1} - \bx^{r}\Big\rangle\right]}_{\cT_3}\\
=\,& {\cT_1} +  
{\cT_2} + {\cT_3},
\end{aligned}
\end{equation}
and by Lemma~\ref{Lemma:x-update-optimal-confidtion},
\begin{equation}\label{eq:lowerbound-T1-T2-T3-stoc}
\begin{aligned}
    \cT_1 \geq \,& \E\left[{\Phi(\bx^{r}, \blambda^{r+1})} - {\Phi(\bx, \blambda^{r+1})}\right] + \frac{\mu_{\bx}}{4}\E\left[\|\bx^{r}-\bx\|^{2}\right] - L_{\bx\bx}{\mathcal{E}_{r}},\\
    \cT_2 \geq \,& \E\left[{\Phi(\bx^{r+1}, \blambda^{r+1})} - {\Phi(\bx^{r}, \blambda^{r+1})}\right]  - 2L_{\bx\bx}\E\left[\|\bx^{r+1}-\bx^{r}\|^{2}\right] - 2L_{\bx\bx}\mathcal{E}_{r},\\
    \cT_3 \geq \, &-\frac{2\chi\,\tau_r}{J}\zeta^{2} - \frac{1}{4\tau_r}\E\left[\mathrm{D}(  \bx^{r+1}, \bx^{r})\right].
\end{aligned}
\end{equation}
Therefore, by combining Eq.~\eqref{eq:apd-1-stoc} and Eq.~\eqref{eq:lowerbound-T1-T2-T3-stoc}, we have
\begin{small}
\begin{equation}
\begin{aligned}
    &  \E\left[{\Phi(\bx^{r+1}, \blambda)} -  {\Phi(\bx, \blambda^{r+1})} \right]  \\ 
    \leq \,& \E\left[{\Phi(\bx^{r+1}, \blambda)} - {\Phi(\bx^{r+1}, \blambda^{r+1})} + {\Phi(\bx^{r+1}, \blambda^{r+1})} - {\Phi(\bx^{r}, \blambda^{r+1})}\right] - \cT_2 - \cT_3  \\
    & + \frac{1}{\tau_r}\E\left[\mathrm{D}(\bx, \bx^{r})-\mathrm{D}(\bx, \bx^{r+1})-\mathrm{D}(\bx^{r+1}, \bx^{r})\right] - \frac{\mu_{\bx}}{4}\E\left[\|\bx^{r}-\bx\|^{2}\right]  + L_{\bx\bx}{\mathcal{E}_{r}} \\
    \leq\,& \E\left[{\Phi(\bx^{r+1}, \blambda)} - {\Phi(\bx^{r+1}, \blambda^{r+1})}\right]    + \frac{1}{\tau_r}\E\left[\mathrm{D}(\bx, \bx^{r})-\mathrm{D}(\bx, \bx^{r+1})-\frac{1}{2}\mathrm{D}(\bx^{r+1}, \bx^{r})\right] \\
    &- \frac{\mu_{\bx}}{4}\E\left[\|\bx^{r}-\bx\|^{2}\right] + 2L_{\bx\bx} \E\left[\|\bx^{r+1}-\bx^{r}\|^{2}\right] + 3L_{\bx\bx}{\mathcal{E}_{r}} + \frac{2\chi\,\tau_r}{J}\zeta^{2} \\
    \leq\,& \E\left[{\Phi(\bx^{r+1}, \blambda)} - {\Phi(\bx^{r+1}, \blambda^{r+1})}\right] + 3 L_{\bx\bx}\E\left[\|\bx^{r+1}-\bx^{r}\|^{2}\right] + 3L_{\bx\bx}{\mathcal{E}_{r}} + \frac{2\chi\,\tau_r}{J}\zeta^{2} \\
    &+ \underbrace{\frac{1}{\tau_r}\E\left[\mathrm{D}(\bx, \bx^{r})-\mathrm{D}(\bx, \bx^{r+1})- \frac{1}{2}\mathrm{D}(\bx^{r+1}, \bx^{r})\right] - \frac{\mu_{\bx}}{8}\E\left[\|\bx^{r+1}-\bx\|^{2}\right]}_{A_r}  \\
    =\,& \E\left[{\Phi(\bx^{r+1}, \blambda)} - {\Phi(\bx^{r+1}, \blambda^{r+1})}\right]    + 3L_{\bx\bx}\E\left[\|\bx^{r+1}-\bx^{r}\|^{2}\right] + 3L_{\bx\bx}{\mathcal{E}_{r}} + \frac{2\chi\,\tau_r}{J}\zeta^{2}  + {A_r},
\end{aligned}
\end{equation}
\end{small}

where we apply the lower bound of $\cT_2$ and $\cT_3$ (Eq.~\eqref{eq:lowerbound-T1-T2-T3-stoc}) for the second inequality,  and apply $-\|\bx\|^{2} \leq \|\by\|^{2} - \frac{1}{2}\|\bx+\by\|^{2}$ for the third inequality, and apply $\mu_{\bx} \leq L_{\bx\bx}$ for the last inequality. We define $A_r$ as
\begin{equation}
    A_r = \frac{1}{\tau_r}\E\left[\mathrm{D}(\bx, \bx^{r})-\mathrm{D}(\bx, \bx^{r+1})- \frac{1}{2}\mathrm{D}(\bx^{r+1}, \bx^{r})\right] - \frac{\mu_{\bx}}{8}\E\left[\|\bx^{r+1}-\bx\|^{2}\right].
\end{equation}
Next, we apply the concavity of $\Phi(\bx^{r+1}, \cdot)$ and combining the above two steps, for the $\bx$-update we have
\begin{small}
\begin{equation}\label{eq:apd-x-stoc}
\begin{aligned}
    &  \E\left[{\Phi(\bx^{r+1}, \blambda)} -  {\Phi(\bx, \blambda^{r+1})}\right]  \\
    \leq\,& \E[\langle \nabla_{\blambda}\Phi(\bx^{r+1}, \blambda^{r+1}), \blambda-\blambda^{r+1} \rangle]    + {A_r}  + 3L_{\bx\bx}\E\left[\|\bx^{r+1}-\bx^{r}\|^{2}\right] + 3L_{\bx\bx}{\mathcal{E}_{r}}  + \frac{2\chi\,\tau_r}{J}\zeta^{2},
\end{aligned}
\end{equation}    
\end{small}

\vspace{-0.05in}
By combining the inequality of $\blambda$-update (Eq.~\eqref{eq:apd-y-stoc}) and $\bx$-update (Eq.~\eqref{eq:apd-x-stoc}), we can get
\begin{equation}\label{eq:lemma-scc-eq2-stoc}
\begin{aligned}
&  \E\left[{F(\bx^{r+1}, \blambda)} -  {F(\bx, \blambda^{r+1})}\right]  \\ 
=\,&  \E\left[\left({\Phi(\bx^{r+1}, \blambda)} - \psi(\blambda)\right) -  \left({\Phi(\bx, \blambda^{r+1})} - \psi(\blambda^{r+1})\right)\right] \\ 
\leq\, & \E\left[\langle \nabla_{\blambda}\Phi(\bx^{r+1}, \blambda^{r+1}), \blambda-\blambda^{r+1} \rangle\right] + \langle\bs^{r}, \blambda^{r+1} - \blambda\rangle + {A_r} + B_r  \\
&\quad+ 3L_{\bx\bx}{\mathcal{E}_{r}} + 3L_{\bx\bx}\E\left[\|\bx^{r+1}-\bx^{r}\|^{2}\right]  + \frac{2\chi\,\tau_r}{J}\zeta^{2} \\
= & -\E\left[\langle \bq^{r+1}, \blambda^{r+1} - \blambda \rangle\right] + \theta_{r}\E\left[\langle\bq^{r}, \blambda^{r+1} - \blambda\rangle\right] + {A_r} + B_r \\
&\quad+ 3L_{\bx\bx}{\mathcal{E}_{r}} + 3L_{\bx\bx}\E\left[\|\bx^{r+1}-\bx^{r}\|^{2}\right]  + \frac{2\chi\,\tau_r}{J}\zeta^{2} \\
= & -\langle \E\left[\bq^{r+1}, \blambda^{r+1} - \blambda \rangle\right] + \theta_{r}\E\left[\langle\bq^{r}, \blambda^{r} - \blambda\rangle\right] + \theta_{r}\E\left[\langle\bq^{r}, \blambda^{r+1} - \blambda^{r}\rangle\right]  \\
&\quad + {A_r} + B_r  + 3L_{\bx\bx}{\mathcal{E}_{r}} + 3L_{\bx\bx}\E\left[\|\bx^{r+1}-\bx^{r}\|^{2}\right]  + \frac{2\chi\,\tau_r}{J}\zeta^{2},
\end{aligned}
\end{equation}
where we apply Eq.~\eqref{eq:apd-x-stoc} for the first inequality, and the definition of $\bq^{r}$ for the second equality, and $\bq^{r}$ is defined as
\begin{equation}
    \bq^r = \nabla_{\blambda}\Phi(\bx^{r}, \blambda^{r}) - \nabla_{\blambda}\Phi(\bx^{r-1}, \blambda^{r-1}).
\end{equation}
The term $\theta^{r}\langle\bq^{r}, \blambda^{r+1} - \blambda^{r}\rangle$ can be upper bounded as
\begin{equation}\label{eq:apd-bound-q-sc-mc-stoc}
\begin{aligned}
    \theta_{r}\langle\bq^{r}, \blambda^{r+1} - \blambda^{r}\rangle 
    &= \theta_{r}\langle\nabla_{\blambda}\Phi(\bx^{r}, \blambda^{r}) - \nabla_{\blambda}\Phi(\bx^{r-1}, \blambda^{r-1}), \blambda^{r+1} - \blambda^{r}\rangle \\
    &=  \theta_{r}\langle\nabla_{\blambda}\Phi(\bx^{r}, \blambda^{r}) - \nabla_{\blambda}\Phi(\bx^{r-1}, \blambda^{r}), \blambda^{r+1} - \blambda^{r}\rangle \\
    &\leq  \frac{\theta_{r}}{2\alpha_{r}}\|\nabla_{\blambda}\Phi(\bx^{r}, \blambda^{r}) - \nabla_{\blambda}\Phi(\bx^{r-1}, \blambda^{r})\|^{2} + 
    \frac{\alpha_{r}\theta_{r}}{2}\|\blambda^{r+1} - \blambda^{r}\|^{2}
\end{aligned}
\end{equation}
where we apply $\nabla_{\blambda}\Phi(\bx^{r-1}, \blambda^{r})=\nabla_{\blambda}\Phi(\bx^{r-1}, \blambda^{r-1})$ (because the $\Phi$ is linear in $\blambda$) for the second equality, and apply the smoothness assumption $$\|\bq^{r}\|= \|\nabla_{\blambda}\Phi(\bx^{r}, \blambda^{r}) - \nabla_{\blambda}\Phi(\bx^{r-1}, \blambda^{r})\| \leq L_{\blambda\bx}\|\bx^{r}-\bx^{r-1}\|$$
in the second inequality. 
Then by combining Eq.~\eqref{eq:apd-bound-q-sc-mc-stoc} and  Eq.~\eqref{eq:lemma-scc-eq2-stoc}, we have
\begin{small}
\begin{equation}\label{eq:lemma1-eq3-stoc}
\begin{aligned}
& \E \left[{F(\bx^{r+1}, \blambda)} -  {F(\bx, \blambda^{r+1})} \right] \\ 
\leq & -\underbrace{\E\left[\langle \bq^{r+1}, \blambda^{r+1} - \blambda \rangle + \frac{1}{\sigma_r}\mathrm{D}(\blambda, \blambda^{r+1}) + \frac{1}{\tau_r}\mathrm{D}(\bx, \bx^{r+1}) + \frac{\mu_{\bx}}{8}\|\bx^{r+1} - \bx\|^{2} + \frac{1}{2\alpha_{r+1}}\|\bq^{r+1}\|^{2} \right]}_{Z_{r+1}} \\ 
&+\underbrace{\E\left[  \theta_{r}\langle\bq^{r}, \blambda^{r} - \blambda\rangle +  \frac{1}{\sigma_r}\mathrm{D}(\blambda, \blambda^{r}) + \frac{1}{\tau_r}\mathrm{D}(\bx, \bx^{r}) + \frac{\theta_{r}}{2\alpha_r}\|\bq^{r}\|^{2} \right]}_{V_r} + \frac{\alpha_{r}\theta_{r}}{2}\E\left[\|\blambda^{r+1}-\blambda^{r}\|^{2}\right]  \\
&+ \E\left[\frac{1}{2\alpha_{r+1}}\|\bq^{r+1}\|^{2} + 3L_{\bx\bx}\|\bx^{r+1}-\bx^{r}\|^{2} + 3L_{\bx\bx}{\mathcal{E}_{r}} - \frac{1}{2\tau_r}\mathrm{D}(\bx^{r+1}, \bx^{r}) - \frac{1}{\sigma_r}\mathrm{D}(\blambda^{r+1}, \blambda^{r})\right] + \frac{2\chi\,\tau_r}{J}\zeta^{2} \\
\leq & -Z_{r+1} + V_r  + \frac{\alpha_{r}\theta_{r}}{2}\E\left[\|\blambda^{r+1}-\blambda^{r}\|^{2}\right] + \frac{L_{\blambda\bx}^{2}}{2\alpha_{r+1}}\E\left[\|\bx^{r+1}-\bx^{r}\|^{2}\right]    + 3L_{\bx\bx}\left[\|\bx^{r+1}-\bx^{r}\|^{2}\right] \\
&+ 3L_{\bx\bx}{\mathcal{E}_{r}} - \frac{1}{2\tau_r}\E\left[\mathrm{D}(\bx^{r+1}, \bx^{r})\right] - \frac{1}{\sigma_r}\left[\mathrm{D}(\blambda^{r+1}, \blambda^{r})\right] + \frac{2\chi\,\tau_r}{J}\zeta^{2} \\
= & -Z_{r+1} + V_r + \left(\frac{\alpha_{r}\theta_{r}}{2} - \frac{1}{2\sigma_r}\right)\E\left[\|\blambda^{r+1}-\blambda^{r}\|^{2}\right] + \left(\frac{L_{\blambda\bx}^{2}}{2\alpha_{r+1}} + 3L_{\bx\bx} - \frac{1}{4\tau_{r}}\right)\E\left[\|\bx^{r+1}-\bx^{r}\|^{2} \right] \\
& + \frac{2\chi\,\tau_r}{J}\zeta^{2}  + \underbrace{3L_{\bx\bx}{\mathcal{E}_{r}} - \frac{1}{8\tau_{r}}\E\left[\|\bx^{r+1}-\bx^{r}\|^{2}\right]}_{\cT_4} .
\end{aligned}
\end{equation}
\end{small}

Next, to get the upper bound of $\cT_4$, we apply Lemma~\ref{Lemma:x-update-lowerbound-drift} to analyze the term $\frac{1}{8\tau_{r}}\E\left[\|\bx^{r+1}-\bx^{r}\|^{2}\right]$,
\begin{equation}
\begin{aligned}
\frac{1}{8\tau_{r}}\E\left[\|\bx^{r+1}-\bx^{r}\|^{2}\right] 
\geq 
-\frac{\tau_r L_{\bx\bx}^{2}}{8} {\mathcal{E}_r} + \frac{\tau_r}{16}\E\left[\left\|\nabla_{\bx}\Phi(\bx^{r}, \blambda^{r+1})\right\|^{2}\right].
\end{aligned}
\end{equation}
By applying Lemma~\ref{Lemma:drift-lemma-deterministic}, we can upper bound the drift error as follows,
\begin{equation}
    \mathcal{E}_r \leq \frac{12\tau_r^{2}}{\eta_{g}^{2}}\E\left[\left\|\nabla_{\bx}\Phi(\bx^{r}, \blambda^{r+1})\right\|^{2}\right] + \left[\frac{8\tau_r^{2}}{\eta_{g}^{2}}\left(1+\chi\right) + \frac{3\tau_r^{2}}{\eta_g^{2} J}\right]\zeta^{2},
\end{equation}
Then if we set the effective step size as $\tau_r = O(1/L_{\bx\bx}))$, the term $\cT_4$ can be upper bounded as
\begin{small}
\begin{equation}\label{eq:apd-scaffold-err-control-stoc}
\begin{aligned}
\cT_4 &= 3L_{\bx\bx}{\mathcal{E}_{r}} - {\frac{1}{8\tau_{r}}\E\left[\|\bx^{r+1}-\bx^{r}\|^{2}\right]} \\
&\leq \left(3L_{\bx\bx}+\frac{\tau_r L_{\bx\bx}^{2}}{8} \right){\mathcal{E}_r} - \frac{\tau_r}{16}\E\left[\left\|\nabla_{\bx}\Phi(\bx^{r}, \blambda^{r+1})\right\|^{2}\right]\\
&\leq \underbrace{\left(\frac{36 \tau_r^{2} L_{\bx\bx}}{\eta_{g}^{2}}+\frac{2\tau_r^{3} L_{\bx\bx}^{2}}{\eta_{g}^{2}} - \frac{\tau_r}{16} \right)}_{\leq 0}\E\left[\left\|\nabla_{\bx}\Phi(\bx^{r}, \blambda^{r+1})\right\|^{2}\right] \\
& \quad\,+\left(3L_{\bx\bx}+\frac{\tau_r L_{\bx\bx}^{2}}{8} \right)\left[\frac{8\tau_r^{2}}{\eta_{g}^{2}}\left(1+\chi\right) + \frac{3\tau_r^{2}}{\eta_g^{2} J}\right]\zeta^{2} \\
&\leq 12\tau_r\left(3\left(1+\chi\right) + \frac{1}{J}\right)\zeta^{2} \leq C\tau_r \zeta^{2},
\end{aligned}
\end{equation}
\end{small}
where $C \geq 0$ is a non-negative constant. Then by combining Eq.~\eqref{eq:apd-scaffold-err-control-stoc} and Eq.~\eqref{eq:lemma1-eq3-stoc}, we have
\begin{equation}
\begin{aligned}
&  \E\left[{L(\bx^{r+1}, \blambda)} -  {L(\bx, \blambda^{r+1})}\right]  \\ 
\leq & -Z_{r+1} + V_r + \frac{2\chi\,\tau_r}{J}\zeta^{2}   + C \tau_r \zeta^{2} \\
&+\underbrace{\left(\frac{\alpha_{r}\theta_{r}}{2} - \frac{1}{2\sigma_r}\right)\E\left[\|\blambda^{r+1}-\blambda^{r}\|^{2}\right] + \left(\frac{L_{\blambda\bx}^{2}}{2\alpha_{r+1}} + 3L_{\bx\bx} - \frac{1}{4\tau_{r}}\right)\E\left[\|\bx^{r+1}-\bx^{r}\|^{2}\right]}_{\Delta_r} \\
= & -Z_{r+1} + V_r + {\Delta_r}  + C\tau_r \zeta^{2},
\end{aligned}
\end{equation}
which completes the proof.
\end{proof}

\subsubsection{How to Set Parameters in Strongly-convex-concave (SC-C) Setting?}\label{subsec:appendix-proof-sc-c-parameters}

Next we study how to set the parameters of Algorithm~\ref{Algorithm:ScaffoldAPD} in the strongly-convex-concave setting.

\begin{lemma}\label{lemma:appendix-scc-stepsize-1}
In Algorithm~\ref{Algorithm:ScaffoldAPD}, if we set the parameters as
\begin{equation}
\begin{aligned}
    \sigma_{-1} &= \gamma_0 \bar{\tau},\quad \sigma_r = \gamma_r \tau_r, \quad
    \theta_r = \sigma_{r-1}/\sigma_r,\quad
    \gamma_{r+1} = \gamma_r (1 + \mu_{\bx}\tau_{r}),
\end{aligned}
\end{equation}
and we set $t_r$ as
\begin{equation}
    t_r = \sigma_r / \sigma_0,
\end{equation}
then we have
\begin{equation}\label{eq:appendix-lemma-stepsize-scc-eq1}
    t_r \left(\frac{1}{\tau_r} + {\mu_{\bx}}\right) \geq \frac{t_{r+1}}{\tau_{r+1}}, \quad \frac{t_r}{\sigma_r} \geq \frac{t_{r+1}}{\sigma_{r+1}}, \quad \frac{t_r}{t_{r+1}} = \theta_{r+1}.
\end{equation}
\end{lemma}
\begin{proof}
Because we have $t_r = \sigma_r / \sigma_0$, then $t_r \left(\frac{1}{\tau_r} + \mu_{\bx}\right) \geq \frac{t_{r+1}}{\tau_{r+1}}$ can be written as
\begin{equation}
    \left(1 + {\tau_r\mu_{\bx}}\right) \geq \frac{\tau_{r}}{\tau_{r+1}}\frac{t_{r+1}}{t_r} = \frac{\tau_{r}}{\tau_{r+1}}\frac{\sigma_{r+1}}{\sigma_r},
\end{equation}
then due to the updates of $\gamma_r$ ($\gamma_{r+1} = \gamma_r (1+\tau_r \mu_{\bx})$) and update of $\sigma_r$ ($\sigma_r = \gamma_r \tau_r$), we have
\begin{equation}
    \left(1 + {\tau_r\mu_{\bx}}\right)  = \frac{\gamma_{r+1}}{\gamma_r} = \frac{\sigma_{r+1}}{\tau_{r+1}} \frac{\tau_{r}}{\sigma_r},
\end{equation}
therefore, the three inequalities in Eq.~\eqref{eq:appendix-lemma-stepsize-scc-eq1} are satisfied. 
\end{proof}

\begin{lemma}\label{lemma:appendix-scc-stepsize-2}
For Algorithm~\ref{Algorithm:ScaffoldAPD}, we have
\begin{equation}
    \frac{\tau_r}{\sigma_r} = \frac{1}{\gamma_r} = O\left(\frac{1}{r^{2}}\right), \quad \gamma_r = O\left(r^{2}\right), \quad \sigma_r = O(r), \quad \tau_r\sigma_r = \tau_0^{2}\gamma_0.
\end{equation}
\end{lemma}
\begin{proof}
Since $\tau_{r+1} = \tau_{r}\sqrt{\gamma_r / \gamma_{r+1}}$, then we have $\tau_r = \tau_0 \sqrt{\gamma_0 / \gamma_{r}}$, then based on the update rule for $\gamma_r$ ($\gamma_{r+1} = \gamma_r (1+{\mu_{\bx}}\tau_r)$), we have
\begin{equation}
\gamma_{r+1} = \gamma_r (1+{\mu_{\bx}}\tau_r) = \gamma_{r} + {\mu_{\bx}}\tau_0\sqrt{\gamma_0 \gamma_r}.
\end{equation}
Then we apply induction to prove that
\begin{equation}
    \gamma_r \geq \frac{ \mu_{\bx}^{2}\tau_0^2\gamma_0}{9}r^{2}.
\end{equation}
Therefore, for $\sigma_r$, we have
\begin{equation}
    \sigma_r = \gamma_r \tau_r = \frac{\gamma_{r+1} - \gamma_r}{{\mu_{\bx}}} \geq \tau_0\sqrt{\gamma_0 \gamma_r} \geq \frac{{\mu}\tau_0^{2}\gamma_0}{3}r,
\end{equation}
and 
\begin{equation}
    \tau_r \sigma_r = \frac{\sigma_r^{2}}{\gamma_r} = \frac{(\gamma_{r+1} - \gamma_r)^2}{\mu_{\bx}^{2} \gamma_r} = \tau_0^2 \gamma_0 = \text{constant},
\end{equation}
furthermore,  we have
\begin{equation}
    \frac{\tau_r}{\sigma_r} = \frac{1}{\gamma_r} = O\left(\frac{1}{r^{2}}\right).
\end{equation}
\end{proof}

\begin{remark}
For the sake of simplicity, we establish the validity of the aforementioned two lemmas by considering the case where the parameter $({1}/{\tau_r} + {\mu_{\bx}})$ is used. 
It is worth noting that in subsequent proofs (Theorem~\ref{thm:appendix-sc-c}), it suffices to substitute a smaller value of $\mu_{\bx}$, such as $({1}/{\tau_r} + {\mu_{\bx}}/4)$. 
\end{remark}

\begin{proposition}
If we first set $\theta_0=1$, then we set $\tau_r, \sigma_r, \theta_r$ such that 
\begin{equation}\label{eq:apd-stepsize-1-stoc}
    \frac{1-\delta}{2\tau_r} \geq 6L_{\bx\bx} + \frac{L_{\blambda\bx}^{2}}{\alpha_{r+1}},\quad \frac{1-\delta}{\sigma_r} \geq \theta_{r}\alpha_r,
\end{equation}
where $\delta \in (0,1)$. Then $\Delta_r \leq 0$ for $r=1, \dots, R$\,.
\end{proposition}

\subsubsection{Convergence Analysis}\label{subsec:appendix-proof-sc-c-theorem-proof}
Finally, we prove the convergence of Algorithm~\ref{Algorithm:ScaffoldAPD} in the strongly-convex-concave setting.
\begin{theorem}\label{thm:appendix-sc-c}
    Under the assumptions of Theorem~\ref{thm:main-sc-c}, 
    Algorithm~\ref{Algorithm:ScaffoldAPD} will converge to $\bx^{\star}$, and 
    \begin{equation}
        \E\left[\|\bx^{R+1} - \bx^{\star}\|^{2} \right] 
        \leq \frac{C_1}{R^{2}}\left[\|\bx^{\star} - \bx^{0}\|^{2} + \| \blambda^{0} - \blambda^{\star}\|^{2}\right] + \frac{C_2}{R}\zeta^{2},
    \end{equation}
where $C_1, C_2 > 0$ are constants.
\end{theorem}

\begin{proof}
For $\theta_r$, we have
\begin{equation}
    \theta_{r+1}= \frac{\sigma_r}{\sigma_{r+1}} = \frac{\tau_r\gamma_r}{\tau_{r+1}\gamma_{r+1}} = \sqrt{\frac{\gamma_{r}}{\gamma_{r+1}}} = \frac{1}{\sqrt{1+{\mu_{\bx}}\tau_r}}, \quad \tau_{r+1} = \tau_{r}\sqrt{\frac{\gamma_r}{\gamma_{r+1}}} = \theta_{r+1} \tau_r,
\end{equation}
where we apply the fact that $\tau_{r+1} = \tau_{r}\sqrt{\gamma_r / \gamma_{r+1}}$. Next we set $t_r, \alpha_r$ as
\begin{equation}
    t_r = \sigma_r / \sigma_0, \quad \alpha_r = c_{\alpha}/\sigma_{r-1},
\end{equation}
where $c_{\alpha}\in(0, 1)$ is a constant. 
then Eq.~\eqref{eq:apd-stepsize-1-stoc} can be written as 
\begin{equation}\label{eq:appendix-thm-scc-stepsize-Delta}
    \frac{1-\delta}{\tau_r} \geq 12 L_{\bx\bx}+ \frac{2L_{\blambda\bx}^{2}\sigma_r}{c_{\alpha}}, \quad 1 - (\delta +  c_\alpha) \geq 0,
\end{equation}
the second one can be easily satisfied, the first one we apply induction to prove it,
\begin{equation}
    \frac{1-\delta}{\tau_{r+1}} = \frac{1-\delta}{\tau_r}\sqrt{\frac{\gamma_{r+1}}{\gamma_{r}}} \geq \left(12 L_{\bx\bx}+ \frac{2 L_{\blambda\bx}^{2}\sigma_r}{c_{\alpha}}\right)\sqrt{\frac{\gamma_{r+1}}{\gamma_{r}}}\geq \left(12 L_{\bx\bx}+ \frac{2 L_{\blambda\bx}^{2}\sigma_{r+1}}{c_{\alpha}}\right),
\end{equation}
where we apply the fact that 
\begin{equation}
    \gamma_{r+1}/\gamma_{r} \geq 1, \quad \sigma_{r+1} = \sigma_r\sqrt{\gamma_{k+1}/\gamma_{k}}.
\end{equation}
Therefore, by Eq.~\eqref{eq:appendix-thm-scc-stepsize-Delta}, we can prove that
\begin{equation*}
    \Delta_r = \E\left[ \left(\frac{\alpha_{r}\theta_{r}}{2} - \frac{1}{2\sigma_r}\right)\|\blambda^{r+1}-\blambda^{r}\|^{2} + \left(\frac{L_{\blambda\bx}^{2}}{2\alpha_{r+1}} + 3L_{\bx\bx} - \frac{1}{4\tau_{r}}\right)\|\bx^{r+1}-\bx^{r}\|^{2}\right] \leq 0.
\end{equation*}
Meanwhile, given the parameters of Algorithm~\ref{Algorithm:ScaffoldAPD} satisfy Condition~\ref{condition:stepsize-scc}, by Lemma~\ref{lemma:appendix-scc-stepsize-1}, we have 
\begin{equation}
    t_{r+1} V_{r+1} \leq t_{r}Z_{r+1}.
\end{equation}
Then, by multiplying Eq.~\eqref{eq:lemma:sc-c-stochastic-eq0} and summing up from $r=0, \cdots, R$, we have
\begin{equation}
\begin{aligned}
&\left[{\sum_{r=0}^{R} t_r }\right] \E\left[F(\bar{\bx}^{R+1}, {\blambda}) - F({\bx}, \bar{\blambda}^{R+1})\right]    +\frac{t_R}{\tau_{R}}\cdot \E\left[\mathrm{D}(\bx^{R+1}, \bx^{\star})\right] \\ 
\leq& \frac{t_0}{\tau_0}\mathrm{D}(\bx^{\star}, \bx^{0}) + \frac{t_0}{\sigma_0}\mathrm{D}(\blambda^{\star}, \blambda^{0}) + \sum_{r=0}^{R}(t_r\cdot C\tau_r \zeta^{2}),
\end{aligned}
\end{equation}
where we defined $\bar{\bx}^{R+1}, \bar{\blambda}^{R+1}$ as
\begin{equation}
    \bar{\bx}^{R+1} = \frac{1}{\sum_{r=0}^{R}t_r}\sum_{r=0}^{R}t_r \bx^{r}, \quad 
    \bar{\blambda}^{R+1} = \frac{1}{\sum_{r=0}^{R}t_r}\sum_{r=0}^{R}t_r \blambda^{r}, 
\end{equation}
because by Lemma~\ref{lemma:appendix-scc-stepsize-2}, we have
\begin{equation}
    \sigma_{R}/\tau_{R} = O(R^{2}), \quad \sum_{r=0}^{R}t_r=O(R^{2}),\quad t_R = \sigma_{R}/\sigma_0, \quad t_r\tau_r=\tau_0^{2} \gamma_0,
\end{equation} 
then we have
\begin{equation}
    \E\left[\mathrm{D}(\bx^{R+1}, \bx^{\star})\right] \leq \frac{C_1}{R^{2}}\left[\frac{t_0}{\tau_0}\mathrm{D}(\bx^{\star}, \bx^{0}) + \frac{t_0}{\sigma_0}\mathrm{D}(\blambda^{\star}, \blambda^{0})\right] + \frac{C_2}{R}\zeta^{2},
\end{equation}
where we apply the fact that 
\begin{equation}
    F(\bar{\bx}^{R+1}, {\blambda}^{\star}) - F({\bx}^{\star}, \bar{\blambda}^{R+1}) \geq 0.
\end{equation}
This completes our proof.
\end{proof}

\newpage

\subsection{Proofs -- strongly-convex-strongly-concave (SC-SC) setting}\label{subsec:appendix-proof-sc-sc}
In this subsection, we first present the technical lemma in Section~\ref{subsec:appendix-proof-sc-sc-lemmas}. 
Next, we analyze how to set the step size related parameters in Section~\ref{subsec:appendix-proof-sc-sc-parameters}.
Finally, we prove Theorem~\ref{thm:main-sc-sc} in Section~\ref{subsec:appendix-proof-sc-sc-theorem-proof}.

\subsubsection{Technical  Lemmas}\label{subsec:appendix-proof-sc-sc-lemmas}

\begin{lemma}\label{lemma:sc-sc-stochastic}
If we set the step size in Alg~\ref{Algorithm:ScaffoldAPD} as $\tau \cdot L_{\bx\bx} \leq 1$, 
then for any $\bx, \blambda$ we have
\begin{equation}\label{eq:lemma-sc-sc-0}
\begin{aligned}
\E\left[{F(\bx^{r+1}, \blambda)} -  {F(\bx, \blambda^{r+1})}\right] \leq -Z_{r+1} + V_r + \Delta_r + C\tau\zeta^{2},
\end{aligned}
\end{equation}
where $Z_{r}, V_{r}, \Delta_r$ are defined as
\begin{small}
\begin{align}
Z_{r+1} &= \E\left[\langle \bq^{r+1}, \blambda^{r+1} - \blambda \rangle + \left(\frac{1}{2\sigma}+\frac{\mu_{\blambda}}{2}\right)\|\blambda^{r+1} - \blambda\|^{2} + \left(\frac{1}{2\tau}+\frac{\mu_{\bx}}{8}\right)\|\bx^{r+1} -\bx\|^{2}\right], \nonumber\\
V_{r} &=   \E\left[\theta^{r}\langle\bq^{r}, \blambda^{r} - \blambda\rangle +  \frac{1}{2\sigma}\|\blambda^{r} - \blambda\|^{2} + \frac{1}{2\tau}\|\bx^{r} - \bx\|^{2}\right], \\
\Delta_r &= \E\left[\left(\frac{\theta L_{\blambda\bx}}{2\pi} - \frac{1}{2\sigma}\right)\|\blambda^{r+1}-\blambda^{r}\|^{2} + \left(\frac{\pi \theta L_{\blambda\bx}}{2} + 3L_{\bx\bx} - \frac{1}{4\tau}\right)\|\bx^{r+1}-\bx^{r}\|^{2}\right],\nonumber
\end{align}
\end{small}
where $\pi>0$ is a parameter and $C \geq 0$ is a constant.

\end{lemma}

\begin{proof} 
Most of the steps are the same as in the Lemma~\ref{lemma:sc-c-stochastic}. To start with, based on the condition that $\psi(\blambda)$ is strongly convex in $\blambda$, we apply Lemma~\ref{Lemma:lamba-update-optimal-confidtion},
\begin{equation}
\begin{aligned}
    &\psi(\blambda^{r+1}) - \langle\bs^{r}, \blambda^{r+1} - \blambda\rangle \\
    \leq &\psi(\blambda^r) + {\frac{1}{\sigma}\left[\mathrm{D}(\blambda, \blambda^{r})-\mathrm{D}(\blambda, \blambda^{r+1})-\mathrm{D}(\blambda^{r+1}, \blambda^{r})\right]} - \frac{\mu_{\blambda}}{2}\|\blambda^{r+1}-\blambda\|^{2}.
\end{aligned}
\end{equation}
Next, we change the way we upper bound $\theta\langle\bq^{r}, \blambda^{r+1} - \blambda^{r}\rangle $ in the strongly-convex-concave setting, and we upper bound this term as follows,
\begin{equation}\label{eq:apd-bound-q-sc-sc}
\begin{aligned}
    \theta\langle\bq^{r}, \blambda^{r+1} - \blambda^{r}\rangle 
    &= \theta \langle\nabla_{\blambda}\Phi(\bx^{r}, \blambda^{r}) - \nabla_{\blambda}\Phi(\bx^{r-1}, \blambda^{r-1}), \blambda^{r+1} - \blambda^{r}\rangle \\
    &=  \theta \langle\nabla_{\blambda}\Phi(\bx^{r}, \blambda^{r}) - \nabla_{\blambda}\Phi(\bx^{r-1}, \blambda^{r}), \blambda^{r+1} - \blambda^{r}\rangle \\
    &\leq  {\theta}\|\nabla_{\blambda}\Phi(\bx^{r}, \blambda^{r}) - \nabla_{\blambda}\Phi(\bx^{r-1}, \blambda^{r})\|\|\blambda^{r+1} - \blambda^{r}\| \\
    &\leq  \frac{\pi\theta L_{\blambda\bx}}{2}\|\bx^{r} - \bx^{r-1}\|^{2} + 
    \frac{\theta L_{\blambda\bx}}{2\pi }\|\blambda^{r+1} - \blambda^{r}\|^{2},
\end{aligned}
\end{equation}
where $\pi > 0$ is a constant. Then we have
\begin{small}
\begin{equation*}
\begin{aligned}
&  {F(\bx^{r+1}, \blambda)} -  {F(\bx, \blambda^{r+1})}  \\ 
\leq & -\underbrace{\E\left[\langle \bq^{r+1}, \blambda^{r+1} - \blambda \rangle  + \frac{1}{\tau}\mathrm{D}(\bx, \bx^{r+1}) + \frac{\mu_{\bx}}{8}\|\bx^{r+1} - \bx\|^{2} + \frac{1}{\sigma}\mathrm{D}(\blambda, \blambda^{r+1}) + \frac{\mu_{\blambda}}{2}\|\blambda^{r+1} - \blambda\|^{2}  \right]}_{Z_{r+1}} \\ 
&+\underbrace{\E\left[  \theta\langle\bq^{r}, \blambda^{r} - \blambda\rangle +  \frac{1}{\sigma}\mathrm{D}(\blambda, \blambda^{r}) + \frac{1}{\tau}\mathrm{D}(\bx, \bx^{r})  \right]}_{V_r}   +\frac{\pi\theta L_{\blambda\bx}}{2}\E\left[\|\bx^{r} - \bx^{r-1}\|^{2}\right] + 
    \frac{\theta L_{\blambda\bx}}{2\pi }\E\left[\|\blambda^{r+1} - \blambda^{r}\|^{2}\right] \\
& + 3L_{\bx\bx}\E\left[\|\bx^{r+1}-\bx^{r}\|^{2}\right] + 3L_{\bx\bx}{\mathcal{E}_{r}} - \frac{1}{2\tau}\E\left[\mathrm{D}(\bx^{r+1}, \bx^{r})\right] - \frac{1}{\sigma}\E\left[\mathrm{D}(\blambda^{r+1}, \blambda^{r})\right] + \frac{2\chi\,\tau}{J}\zeta^{2}\\
= & -Z_{r+1} + V_r + \left(\frac{\theta L_{\blambda\bx}}{2\pi} - \frac{1}{2\sigma}\right)\E\left[\|\blambda^{r+1}-\blambda^{r}\|^{2}\right] + \left(\frac{\pi \theta L_{\blambda\bx}}{2} + 3L_{\bx\bx} - \frac{1}{4\tau}\right)\E\left[\|\bx^{r+1}-\bx^{r}\|^{2}\right] \\
& + 3L_{\bx\bx}{\mathcal{E}_{r}} - \frac{1}{8\tau}\E\left[\|\bx^{r+1}-\bx^{r}\|^{2} \right]  + \frac{2\chi\,\tau}{J}\zeta^{2} \\
\geq & -Z_{r+1} + V_r + \underbrace{\left(\frac{\theta L_{\blambda\bx}}{2\pi} - \frac{1}{2\sigma}\right)\E\left[\|\blambda^{r+1}-\blambda^{r}\|^{2} \right]+ \left(\frac{\pi \theta L_{\blambda\bx}}{2} + 3L_{\bx\bx} - \frac{1}{4\tau}\right)\E\left[\|\bx^{r+1}-\bx^{r}\|^{2}\right]}_{\Delta_r} \\
&+ C\tau\zeta^{2},
\end{aligned}
\end{equation*}
\end{small}

where the last inequality is because Eq.~\eqref{eq:apd-scaffold-err-control-stoc}.  This completes our proof.
\end{proof}

\subsubsection{How to Set Parameters in Strongly-convex-strongly-concave (SC-SC) Setting?}\label{subsec:appendix-proof-sc-sc-parameters}
\begin{lemma}\label{lemma:appendix-scsc-stepsize-1}
For Algorithm~\ref{Algorithm:ScaffoldAPD}, if we set the parameters as 
\begin{equation}
        \mu_{\bx}\tau = O\left(\frac{1-\theta}{\theta}\right),\quad \mu_{\blambda}\sigma = O\left(\frac{1-\theta}{\theta}\right),\quad \frac{1}{1-\theta} = O\left(\frac{L_{\bx\bx}}{\mu_{\bx}} + \sqrt{\frac{L_{\blambda\bx}^{2}}{\mu_{\bx}\mu_{\blambda}}}\right),
\end{equation}
then we have
\begin{equation}\label{eq:lemma:appendix-scsc-stepsize-eq1}
    \frac{1}{2\tau} + \frac{\mu_{\bx}}{8} \geq \frac{1}{2\tau\theta}, \quad \frac{1}{2\sigma} + \frac{\mu_{\blambda}}{2} \geq \frac{1}{2\sigma\theta}, \quad \frac{1}{\tau} \geq 12L_{\bx\bx} + 2\pi \theta L_{\blambda\bx}, \quad \frac{1}{\sigma} \geq \frac{\theta L_{\blambda\bx}}{\pi}.
\end{equation}
\end{lemma}

\begin{proof}
The conditions in Eq.~\eqref{eq:lemma:appendix-scsc-stepsize-eq1} can be reformulated as follows,
\begin{equation}\label{eq:apd-step-size-condition-sc-sc}
\begin{aligned}
    \frac{1}{2\tau} + \frac{\mu_{\bx}}{8} \geq \frac{1}{2\tau\theta} \quad &\Leftrightarrow \quad \mu_{\bx}\tau \geq 4\frac{1-\theta}{\theta}, \\ 
    \frac{1}{2\sigma} + \frac{\mu_{\blambda}}{2} \geq \frac{1}{2\sigma\theta} \quad &\Leftrightarrow \quad \mu_{\blambda}\sigma \geq \frac{1-\theta}{\theta}, \\ 
    \frac{1}{\tau} \geq 12L_{\bx\bx} + 2\pi \theta L_{\blambda\bx} \quad &\Leftarrow  \quad \frac{1}{\tau} \geq 12L_{\bx\bx} + 2\pi L_{\blambda\bx}, \\ 
    \frac{1}{\sigma} \geq \frac{\theta L_{\blambda\bx}}{\pi} \quad &\Leftarrow  \quad \frac{c}{\sigma} \geq \frac{\theta L_{\blambda\bx}}{\pi}, 
\end{aligned}
\end{equation}
where $c \in (0, 1]$. 

Next we study how to set $\{\tau, \sigma, \theta\}$ such that Eq.~\eqref{eq:apd-step-size-condition-sc-sc} holds, we could set 
\begin{equation}\label{eq:apd-step-size-condition-sc-sc-2}
\begin{aligned}
    \tau &\geq \frac{4}{\mu_{\bx}}\frac{1-\theta}{\theta},\quad
    \sigma \geq \frac{1}{\mu_{\blambda}}\frac{1-\theta}{\theta}, \\
    \pi &= \frac{\theta \sigma L_{\blambda\bx}}{c}, \\
    \frac{\mu_{\bx} \theta}{4(1-\theta)} - 12 L_{\bx\bx} &\geq \frac{2\theta \sigma L_{\blambda\bx}^{2}}{c}\geq (1-\theta)\frac{2  L_{\blambda\bx}^{2}}{c\mu_{\blambda}},
\end{aligned}
\end{equation}
therefore, once $\theta$ satisfy the following condition
\begin{equation}\label{eq:apd-step-size-condition-sc-sc-3}
    \frac{\mu_{\bx} \theta}{4(1-\theta)} - 12 L_{\bx\bx} \geq (1-\theta)\frac{2  L_{\blambda\bx}^{2}}{c\mu_{\blambda}},
\end{equation}
and then we can set $\tau$ and $\sigma$ based on the value of $\theta$ according to Eq.~\eqref{eq:apd-step-size-condition-sc-sc-2}. 
Then if we let 
\begin{equation}
    \omega = \frac{1}{1-\theta},
\end{equation}
therefore, based on Eq.~\eqref{eq:apd-step-size-condition-sc-sc-3}, by setting $c=1$, we have
\begin{equation}\label{eq:appendix-omega-sc-sc}
\begin{aligned}
    &\quad \frac{\omega-1}{\omega} \frac{\omega \mu_{\bx}}{4} - 12 L_{\bx\bx} \geq \frac{1}{\omega} \frac{2  L_{\blambda\bx}^{2}}{\mu_{\blambda}},\\
\Leftrightarrow&\quad    
\mu_{\bx}\mu_{\blambda}\omega^{2} - (\mu_{\bx}\mu_{\blambda} + 48\mu_{\blambda}L_{\bx\bx})\omega  - 8 L_{\blambda\bx}^{2}\geq 0, \\
\Leftarrow&\quad \omega= C_{\omega}\frac{(\mu_{\bx}\mu_{\blambda} + 48 \mu_{\blambda}L_{\bx\bx}) + \sqrt{(\mu_{\bx}\mu_{\blambda} + 48 \mu_{\blambda}L_{\bx\bx})^{2} + 32 \mu_{\bx}\mu_{\blambda}L_{\blambda\bx}^{2}}}{2\mu_{\bx}\mu_{\blambda}},\\
\Leftrightarrow&\quad \omega= C_{\omega}\left(\frac{1}{2} + \frac{24 L_{\bx\bx}}{\mu_{\bx}} + \sqrt{\left(\frac{1}{2} + \frac{24 L_{\bx\bx}}{\mu_{\bx}}\right)^{2} + \frac{16 L_{\blambda\bx}^{2}}{\mu_{\bx}\mu_{\blambda}}}\right), \\
\Leftrightarrow&\quad \omega= O\left(\frac{L_{\bx\bx}}{\mu_{\bx}} + \sqrt{\frac{L_{\blambda\bx}^{2}}{\mu_{\bx}\mu_{\blambda}}}\right),
\end{aligned}
\end{equation}
where $C_{\omega}\geq 1$ is a constant. This completes our proof.

\end{proof}

\subsubsection{Convergence Analysis}\label{subsec:appendix-proof-sc-sc-theorem-proof}

\begin{theorem}
Under the assumptions in Theorem~\ref{thm:main-sc-sc}, Algorithm~\ref{Algorithm:ScaffoldAPD} will converge to $\bx^{\star}$, and
\begin{equation}
\E\left[\|\bx^{r} - \bx^{\star}\|^{2}\right] \leq C_1\theta^{R}\left[\|\bx^{0} - \bx^{\star}\|^{2} + \| \blambda^{0} - \blambda^{\star}\|^{2}\right]+ C_2(1-\theta)\frac{\zeta^{2}}{\mu_{\bx}^{2}},
\end{equation}
where $C_1, C_2 \geq 0$ are non-negative constants.
\end{theorem}
\begin{proof}
The last two conditions in Eq.~\eqref{eq:lemma:appendix-scsc-stepsize-eq1} ensure 
\begin{equation*}
    \Delta_r = \E\left[\left(\frac{\theta L_{\blambda\bx}}{2\pi} - \frac{1}{2\sigma}\right)\|\blambda^{r+1}-\blambda^{r}\|^{2} + \left(\frac{\pi \theta L_{\blambda\bx}}{2} + 3L_{\bx\bx} - \frac{1}{4\tau}\right)\|\bx^{r+1}-\bx^{r}\|^{2}\right] \leq 0,
\end{equation*}
for $r=0,\dots, R$. The first two conditions in Eq.~\eqref{eq:lemma:appendix-scsc-stepsize-eq1} ensure 
\begin{equation*}
\begin{aligned}
Z_{r+1} \geq \frac{1}{\theta} V_{r+1}.
\end{aligned}
\end{equation*}    
Therefore, by applying Lemma~\ref{lemma:sc-sc-stochastic}, we have
\begin{equation}
    \E\left[{F(\bx^{r+1}, \blambda)} -  {F(\bx, \blambda^{r+1})}\right] + \frac{1}{\theta}V_{r+1} \leq  V_r + \Delta_r + C\tau\zeta^{2},
\end{equation}

Then we plug $\bx=\bx^{\star}, \blambda=\blambda^{\star}$ in Eq.~\eqref{eq:lemma-sc-sc-0}, and we have ${F(\bx^{r+1}, \blambda^{\star})} -  {F(\bx^{\star}, \blambda^{r+1})} \geq 0$, 
then we have
\begin{equation}
    V_{r+1} \leq \theta V_r + \theta\Delta_{r} + C\theta\tau\zeta^{2} ,
\end{equation}
therefore, we can derive that
\begin{equation}
    V_{R} \leq {\theta}^{R} V_0 + \theta\Delta_{R} + \frac{C\tau\theta\zeta^{2}}{1-{\theta}},
\end{equation}
meanwhile, we can set the parameters $\{\tau, \sigma, \theta\}$ (according to Eq.~\eqref{eq:apd-step-size-condition-sc-sc}) such that
\begin{equation}
    \E\left[\|\bx^{r} - \bx\|^{2}\right] \leq 4\tau\theta^{R}V_0 + \frac{C\tau^{2}\theta\zeta^{2}}{1-{\theta}}
\end{equation}
since $\tau = 2(1-\theta)/(\theta \mu_{\bx})$,
\begin{equation}
    \E\left[\|\bx^{r} - \bx\|^{2}\right] \leq 4\tau\theta^{R}V_0 + \frac{4C(1-\theta)\zeta^{2}}{\theta\mu_{\bx}^{2}}
\end{equation}

We need to run at least $N_\varepsilon$ rounds such that $4\tau\theta^{R}V_0 + \frac{4C(1-\theta)\zeta^{2}}{\theta\mu_{\bx}^{2}} \leq 2\varepsilon$. 

Suppose $N_\varepsilon$ satisfies
\begin{equation}
    N_\varepsilon = O\left(\ln\left(\frac{V_0}{\varepsilon}\right)\Big/\ln\left(\frac{1}{\theta}\right)\right),
\end{equation}
then we have $4\tau\theta^{R}V_0  \leq \varepsilon$. 
Because $\ln(1/\theta)$ is convex in $\theta\in\mathbb{R}_{+}$, then we have 
\begin{equation}
    \ln\left(\frac{1}{\theta}\right) \leq \frac{1}{1 - \theta}, \quad \theta \in (0, 1),
\end{equation}
therefore, to get an upper bound for $N_\varepsilon$, we only need to get the upper bound for $\frac{1}{1 - \theta}$. Then if we set $\omega = \frac{1}{1-\theta}$,
then based on Eq.~\eqref{eq:appendix-omega-sc-sc}, 
we have
\begin{equation}
\begin{aligned}
\omega= O\left(\frac{L_{\bx\bx}}{\mu_{\bx}} + \sqrt{\frac{L_{\blambda\bx}^{2}}{\mu_{\bx}\mu_{\blambda}}}\right),
\end{aligned}
\end{equation}
meanwhile, we need to ensure $\frac{4C(1-\theta)\zeta^{2}}{\theta\mu_{\bx}^{2}}$ is small, i.e.,
\begin{equation}
    \frac{4C(1-\theta)\zeta^{2}}{\theta\mu_{\bx}^{2}} = \varepsilon \quad \Leftrightarrow \quad \frac{1}{1-\theta} = \frac{4C\zeta^{2}}{\theta\mu_{\bx}^{2}\varepsilon}.
\end{equation}

therefore, in ensure $\E\left[\|\bx^{r} - \bx\|^{2}\right] \leq 2\varepsilon$, the number of communication rounds satisfies
\begin{equation}
    N_\varepsilon = \widetilde{O}\left(\frac{L_{\bx\bx}}{\mu_{\bx}} + \sqrt{\frac{L_{\blambda\bx}^{2}}{\mu_{\bx}\mu_{\blambda}}} + \frac{\zeta^{2}}{\mu_{\bx}^{2} \varepsilon}\right),
\end{equation}
which completes our proof.
\end{proof}

\newpage
\section{Additional Implementation Details and Experimental Results}\label{sec:appendix-exp}
In this section, we provide further details for algorithm implementations (Section~\ref{subsec:appendix-exp-details}) as well as additional experimental results -- trade-off between worst-20\% and average accuracy (Section~\ref{subsec:appendix-additional-exp}), convergence performance on synthetic datasets (Section~\ref{subsec:appendix-exp-synthetic}), and comparison with existing methods (Section~\ref{subsec:appendix-exp-comparison}).

\subsection{Additional Experimental Details}\label{subsec:appendix-exp-details}
In order to enhance the performance of baseline methods, we incorporate local steps into the AFL~\citep{mohri2019agnostic} method. We find that employing local steps yields significantly better performance compared to taking a single gradient step. 
To ensure a fair comparison, we employ identical feature extraction procedures across all methods. Following the setup outlined in \citet{yu2022tct}, we first compute the empirical neural tangent kernel (eNTK) representations of the input samples. Then, we randomly select 50,000 features from the eNTK representation through subsampling. 
For the (local) objective function, we utilize the mean squared error (MSE) loss, which has been used for classification tasks as described in \citet{yu2022tct}. To calculate the average accuracy, we begin by computing the test accuracy of each client. Then, we compute the average accuracy by averaging the results from all clients.

\subsection{Trade-off between Worst-20\% Accuracy and Average Accuracy}\label{subsec:appendix-additional-exp}
We present the trade-off between worst-20\% accuracy and average/best-20\% accuracy through a scatter plot, as illustrated in Figure~\ref{fig:exp-real-scatter}.
We consider the TinyImageNet dataset with the Non-i.i.d. degree parameter $\alpha=0.01$. 
Our proposed algorithm, as illustrated in Figure~\ref{fig:exp-real-scatter}, showcases a compelling trade-off between accuracy in the worst-20\% and the average/best-20\% scenarios.

\begin{figure}[ht]
     \centering
     \begin{subfigure}[b]{0.47\textwidth}
         \centering
    \includegraphics[width=\textwidth]{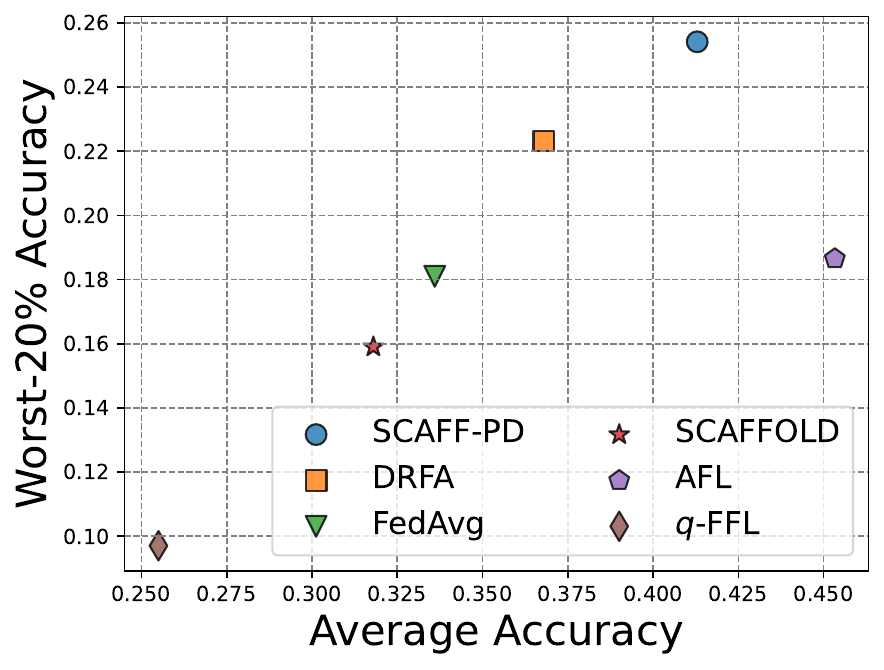}
         \caption{Worst-20\%  v.s. Average accuracy.}
     \end{subfigure}
     \begin{subfigure}[b]{0.47\textwidth}
         \centering
    \includegraphics[width=\textwidth]{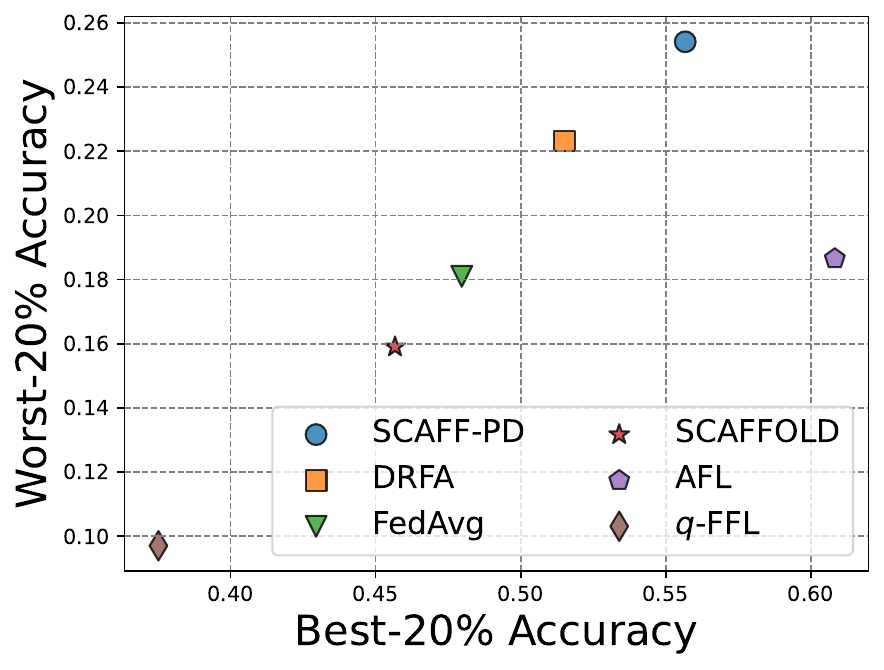}
         \caption{Worst-20\%  v.s. Best-20\% accuracy.}
     \end{subfigure}
        \caption{
        Compare the average/worst-20\%/best-20\% accuracy of different algorithms on TinyImageNet with $\alpha=0.01$.
        }
        \label{fig:exp-real-scatter}
        \vspace{-0.1in}
\end{figure}

\newpage
\subsection{Additional Experiments on Synthetic Datasets}\label{subsec:appendix-exp-synthetic}
We vary the level of data heterogeneity by changing the parameter $\sigma$ from 0.01 to 0.1, where $\sigma$ is used for generating $\delta_{i}^{\bx}$ ($\delta_{i}^{\bx} \sim \mathcal{N}(\bm{0}, \sigma^{2}\bm{I}_{d\times d})$). 
Figure~\ref{fig:appendix-exp-synthetic-heterogeneity} illustrates the fast convergence of {\algname} to the optimal solution across various data heterogeneity settings. 
We also explore the effect of varying the number of local steps. Figure~\ref{fig:appendix-exp-synthetic-local-steps} demonstrates that increasing the number of local steps results in faster convergence towards the optimal solution.

\begin{figure}[ht]
     \centering
     \begin{subfigure}[b]{0.47\textwidth}
         \centering
    \includegraphics[width=\textwidth]{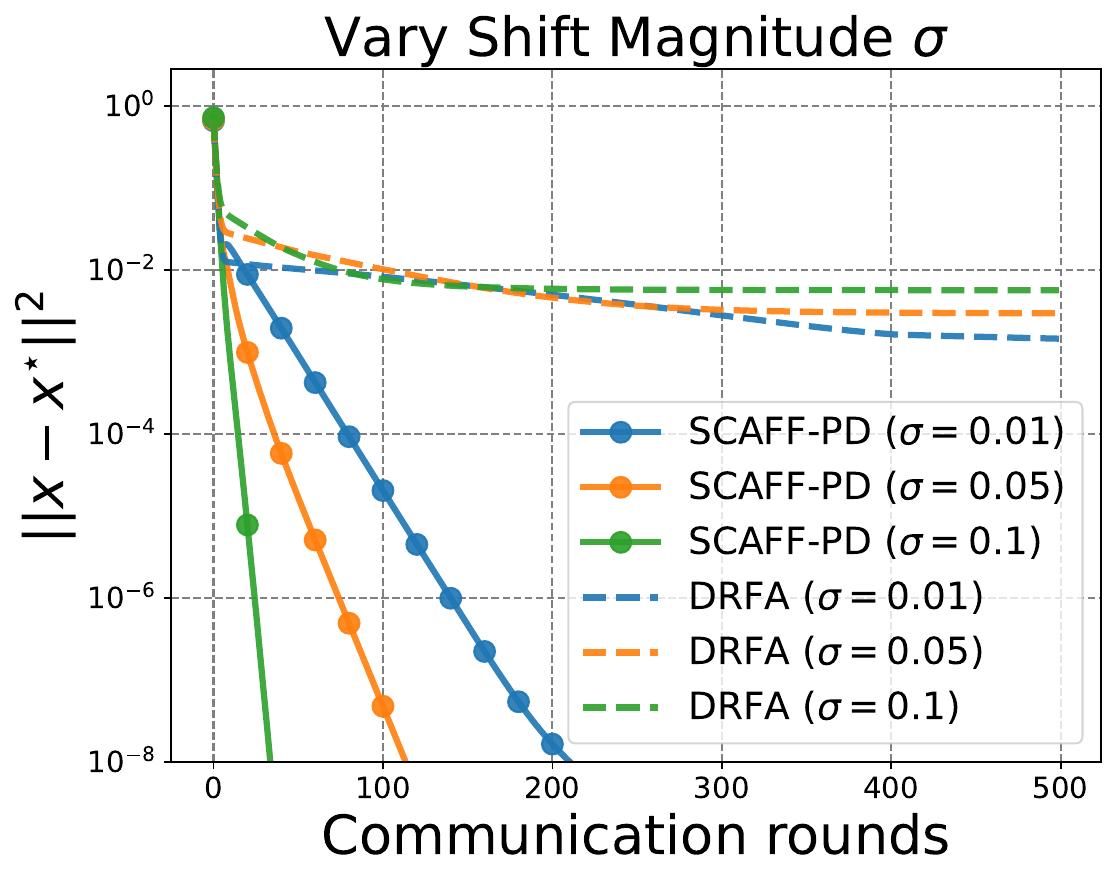}
         \caption{Vary the level of data heterogeneity.}
         \label{fig:appendix-exp-synthetic-heterogeneity}
     \end{subfigure}
     \begin{subfigure}[b]{0.47\textwidth}
         \centering
    \includegraphics[width=\textwidth]{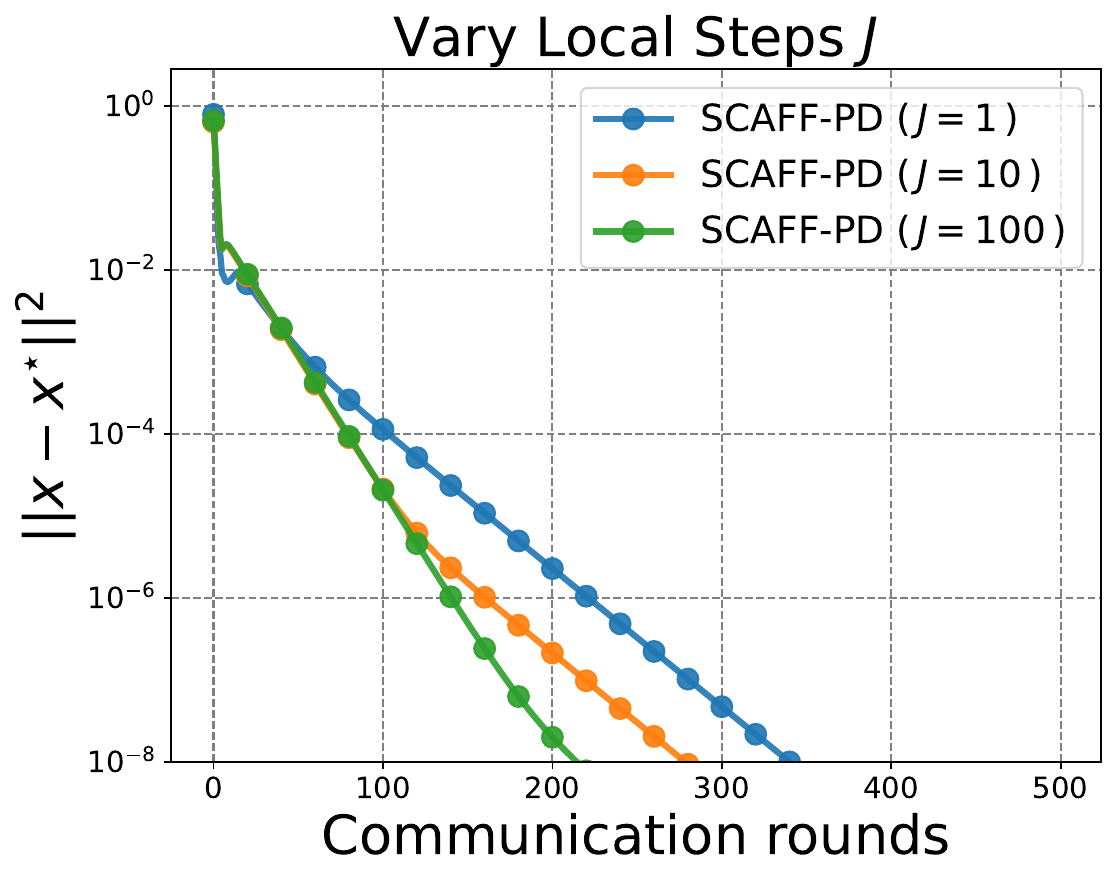}
         \caption{Effect of local steps.}
         \label{fig:appendix-exp-synthetic-local-steps}
     \end{subfigure}
        \caption{
        (\textbf{left}) Compare {\algname} and DRFA under different levels of data heterogeneity. 
        (\textbf{right}) Study the effect of local steps for our proposed algorithm on the synthetic dataset.
        }
        \label{fig:appendix-exp-synthetic}
        \vspace{-0.1in}
\end{figure}

\subsection{Additional Experiments on Comparison with Existing Methods}\label{subsec:appendix-exp-comparison}
\paragraph{More clients.} 
On the CIFAR100 dataset, we conduct a comparison of different algorithms in the 50 clients setting, following the configuration outlined in Table~\ref{tab:main_table}. 
The summarized results are presented in Table~\ref{tab:appendix-table-more-clients}. 
Consistent with our previous findings, {\algname} exhibits superior robustness when compared to existing methods.

\begin{table*}[ht]
	\centering
	\caption{The average and worst-20\% top-1 accuracy of our algorithm ({\algname}) vs.\ state-of-the-art federated learning algorithms evaluated on CIFAR100 with 50 clients. 
 The highest top-1 accuracy in each setting is highlighted in \textbf{bold}.
	}
	\label{tab:appendix-table-more-clients}
	    \footnotesize
     \resizebox{0.5\textwidth}{!}{
		\begin{tabular}{cccccc}
			\toprule
			\multirow{1}{*}{\bf Datasets}  &   \multirow{1}{*}{\bf Methods}                 & \multicolumn{2}{c}{\bf Non-i.i.d. degree}   
			  \\
			  \midrule
			 &                   & \multicolumn{2}{c}{$\alpha=0.01$}    
			  \\
		\midrule
			\multirow{11}{*}{CIFAR-100}   &  &    
			\texttt{average} & \texttt{worst-20\%} 
			\\ 
			\cmidrule(lr){1-4} 
			   & {FedAvg} & 45.45  &  20.64  \\ 
			\tablewhitespace
			   & {SCAFFOLD} & 43.73 & 18.33  \\
			  \tablewhitespace
			   & {$q$-FFL} &  33.42 & 8.13    \\
			  \tablewhitespace
			   & {AFL} & 49.93  &  31.87     \\
			  \tablewhitespace
			   & {DRFA} & \textbf{51.07}  &  31.23  \\
			  \tablewhitespace
                 & \textit{SCAFF-PD}  & {50.43} & \textbf{33.03}     \\
                 \bottomrule
		\end{tabular}
    }
	\vspace{-.75em}
\end{table*}

\paragraph{Additional dataset.} 
We consider another dataset -- CIFAR10 dataset, the setup mostly follows the configuration outlined in Table~\ref{tab:main_table}. 
We  summarize the results  in Table~\ref{tab:appendix-table-more-dataset}. 
We observe that {\algname} outperforms existing methods.

\begin{table*}[ht]
	\centering
	\caption{The average and worst-20\% top-1 accuracy of our algorithm ({\algname}) vs.\ state-of-the-art federated learning algorithms evaluated on CIFAR10 with 20 clients and $\alpha=0.05$. 
 The highest top-1 accuracy in each setting is highlighted in \textbf{bold}.
	}
	\label{tab:appendix-table-more-dataset}
	    \footnotesize
     \resizebox{0.5\textwidth}{!}{
		\begin{tabular}{cccccc}
			\toprule
			\multirow{1}{*}{\bf Datasets}  &   \multirow{1}{*}{\bf Methods}                 & \multicolumn{2}{c}{\bf Non-i.i.d. degree}   
			  \\
			  \midrule
			 &                   & \multicolumn{2}{c}{$\alpha=0.05$}    
			  \\
		\midrule
			\multirow{11}{*}{CIFAR-100}   &  &    
			\texttt{average} & \texttt{worst-20\%} 
			\\ 
			\cmidrule(lr){1-4} 
			   & {FedAvg} & 77.42  &  60.63  \\ 
			\tablewhitespace
			   & {SCAFFOLD} & 77.75 & 62.89  \\
			  \tablewhitespace
			   & {$q$-FFL} &  68.52 & 41.26    \\
			  \tablewhitespace
			   & {AFL} & 78.89  &  65.07     \\
			  \tablewhitespace
			   & {DRFA} & {79.04}  &  65.02  \\
			  \tablewhitespace
                 & \textit{SCAFF-PD}  & \textbf{79.71} & \textbf{69.59}     \\
                 \bottomrule
		\end{tabular}
    }
	\vspace{-.75em}
\end{table*}

\paragraph{Additional baselines.} 
In addition to the baseline methods listed in Table~\ref{tab:main_table}, we include $\Delta$-FL~\citep{pillutla2021federated} and FedProx~\citep{li2020federated} in our evaluation. 
We adopt a similar setup as presented in Table~\ref{tab:main_table} to assess the performance of these two methods. 
The summarized results are presented in Table~\ref{tab:appendix-table-more-baseline}, indicating that our proposed algorithm surpasses both $\Delta$-FL and FedProx in terms of worst-20\% accuracy and average accuracy.

\begin{table*}[ht]
	\centering
	\caption{The average and worst-20\% top-1 accuracy of our algorithm ({\algname}) vs.\ state-of-the-art federated learning algorithms evaluated on CIFAR100 and Tiny-ImageNet. 
 The highest top-1 accuracy in each setting is highlighted in \textbf{bold}.
	}
	\label{tab:appendix-table-more-baseline}
	    \footnotesize
     \resizebox{0.95\textwidth}{!}{
		\begin{tabular}{cccc|cc|ccc}
			\toprule
			\multirow{1}{*}{\bf Datasets}  &   \multirow{1}{*}{\bf Methods}                 & \multicolumn{6}{c}{\bf Non-i.i.d. degree}   
			  \\
			  \midrule
			 &                   & \multicolumn{2}{c}{$\alpha=0.01$}    & \multicolumn{2}{c}{$\alpha=0.05$}    & \multicolumn{2}{c}{$\alpha=0.1$}   
			  \\
		\midrule
			\multirow{6.2}{*}{CIFAR-100}   &  &    
			\texttt{average} & \texttt{worst-20\%} & \texttt{average} & \texttt{worst-20\%} &\texttt{average} & \texttt{worst-20\%} & 
			\\ 
			\cmidrule(lr){1-9} 
			   & {FedProx} & 38.76 & 15.58 & 35.91 & 24.57 & 36.49 & 26.45    \\ 
			\tablewhitespace
			   & {$\Delta$-FL} & 30.09 & 7.26 & 33.18 & 15.82 & 31.69 & 16.63     \\ 
			  \tablewhitespace
                 & \textit{SCAFF-PD}  & \textbf{49.03} & \textbf{29.30} & \textbf{42.06} & \textbf{28.37} & \textbf{43.69} & \textbf{32.77}     \\
                \midrule
   \multirow{6.2}{*}{TinyImageNet}   &  &    
			\texttt{average} & \texttt{worst-20\%} & \texttt{average} & \texttt{worst-20\%} &\texttt{average} & \texttt{worst-20\%} & 
			\\ 
			\cmidrule(lr){1-9} 
			   & {FedProx} & 33.65 & 18.09 & 31.52 & 23.62 & 34.98 & 27.59     \\ 
			\tablewhitespace
			   & {$\Delta$-FL} & 29.06 & 11.94 & 36.77 & 22.24 & 36.47 & 20.13     \\ 
			  \tablewhitespace
                 & \textit{SCAFF-PD} & \textbf{41.26} & \textbf{25.32}    & \textbf{39.32} & \textbf{30.27}  & \textbf{41.23} &  \textbf{29.78}  \\
			\bottomrule
		\end{tabular}
    }
	\vspace{-.75em}
\end{table*}

\end{document}